\documentclass[10pt,twocolumn,twoside]{IEEEtran}

\usepackage{amsmath}
\usepackage{amssymb}
\usepackage{mathabx}
\usepackage{amsthm}
\usepackage{graphicx}
\usepackage{epstopdf}
\usepackage{bm}
\usepackage{cite}
\usepackage{dsfont}
\usepackage{algorithm}
\usepackage[noend]{algpseudocode}
\usepackage{xargs}
\usepackage{tikz}
\usepackage{todonotes}
\usepackage{mathtools}
\usepackage[colorlinks=false]{hyperref}
\allowdisplaybreaks

\newcommandx{\is}[2][1=]{\todo[linecolor=cyan,backgroundcolor=cyan!25,bordercolor=cyan,size=small,author=Iman,#1]{#2}}

\newcommandx{\ff}[2][1=]{\todo[linecolor=magenta,backgroundcolor=magenta!25,bordercolor=magenta,size=small,author=Farhad,#1]{#2}}

\newcommandx{\al}[2][1=]{\todo[linecolor=blue,backgroundcolor=blue!25,bordercolor=blue,size=small,author=Alex,#1]{#2}}

\newtheorem{theorem}{Theorem}[section]
\newtheorem{lemma}[theorem]{Lemma}
\newtheorem{assumption}{Assumption}[section]
\newtheorem{remark}{Remark}[section]
\newtheorem{proposition}[theorem]{Proposition}
\newtheorem{corollary}[theorem]{Corollary}

\DeclareMathOperator{\trace}{trace}
\DeclareMathOperator*{\argmin}{arg\,min}

\begin{document}
\title{Safe Learning of Uncertain Environments}

\author{Farhad Farokhi$^*$, Alex S. Leong$^*$, Iman Shames, and Mohammad Zamani\thanks{$^*$F. Farokhi and A. S. Leong have made equal contributions to this paper.}
  \thanks{F. Farokhi is with the University of Melbourne, e-mail: \texttt{ffarokhi@unimelb.edu.au}. A.~S.~Leong and M.~Zamani are with the Defence Science and Technology Group, Australia, e-mail: \texttt{\{alex.leong,  mohammad.zamani\}@dst.defence.gov.au}. I. Shames is with the Australian National University, e-mail: \texttt{iman.shames@anu.edu.au}.}
  \thanks{The work of F. Farokhi is, in part, supported by a research contract (ID10298 `Safe learning in distributed multi-agent control') from the Defence Science and Technology (DST) Group, Australia. The work of I.~Shames is partially supported by the Australian Government, via grant AUSMURIB000001 associated with ONR MURI grant N00014-19-1-2571. }
}

\maketitle

\begin{abstract}
In many learning based control methodologies, learning the unknown dynamic model precedes the control phase, while the aim is to control the system such that it remains in some safe region of the state space. In this work, our aim is to guarantee safety while learning and control proceed simultaneously. Specifically, we consider the problem of safe learning in nonlinear control-affine systems subject to unknown additive uncertainty. We first model the uncertainty as a Gaussian noise and use state measurements to learn its mean and covariance. We provide rigorous time-varying bounds on the mean and covariance of the uncertainty and employ them to modify the control input via an optimization program with potentially time-varying safety constraints. We show that with an arbitrarily large probability we can guarantee that the state will remain in the safe set, while learning and control are carried out simultaneously, provided that a feasible solution exists for the optimization problem. We provide a secondary formulation of this optimization that is computationally more efficient. This is based on tightening the safety constraints to counter the uncertainty about the learned mean and covariance. The magnitude of the tightening can be decreased as our confidence in the learned mean and covariance increases (i.e., as we gather more measurements about the environment). Extensions of the method are provided for non-Gaussian process noise with unknown mean and covariance as well as Gaussian uncertainties with state-dependent mean and covariance to accommodate more general environments. 
\end{abstract}

\section{Introduction}

Safety of a dynamically controlled system can be defined as guaranteeing that the closed loop system trajectory remains inside a subset of its state space, denoted by the safety set. The Control Barrier Function (CBF) approach and many control strategies achieves this by requiring a reliable dynamic model of the system. However, dynamical models are typically subject to uncertainties. There are many approaches to deal with these uncertainties. Namely, one can model uncertainty as deterministic or stochastic signals with known (upper bound) size or statistics. In these cases, the control design aims to guarantee safety for the worst-case realisation of the uncertainty. In more recent literature, the attempt is to learn the uncertain parts of the dynamics from historic data of the system in order to obtain more accurate and less conservative models of the uncertain parts of the system model. In many such approaches a commonality is that learning the uncertain parameters is carried out separately (prior to control), to ensure an accurate model of the uncertainties is available before attempting to control the system safely. In this paper, we weaken this assumption by learning the uncertainties while maintaining safety through appropriate manipulations of the control signal.

\paragraph*{Related Work}
Safe control using control barrier functions, for systems where model uncertainties are unknown and learnt include \cite{TaylorSingletaryYueAmes,ChengOroszMurrayBurdick,ChengKhojastehAmesBurdick,ChoiCastanedaTomlinSreenath,JagtapPappasZamani}. The works \cite{TaylorSingletaryYueAmes,ChoiCastanedaTomlinSreenath} attempt to learn the form of the uncertainty using data-driven approaches, but no rigorous guarantees on safety are provided. The model uncertainties are modelled using Gaussian processes in \cite{ChengOroszMurrayBurdick,ChengKhojastehAmesBurdick,JagtapPappasZamani}, with the aim of learning the parameters of these Gaussian processes to then derive safety bounds. 
All of these works assume either that the model uncertainty \cite{ChoiCastanedaTomlinSreenath,JagtapPappasZamani} or various model parameters of the Gaussian process \cite{ChengKhojastehAmesBurdick} are learnt before control, or learnt in batch mode \cite{ChengOroszMurrayBurdick,TaylorSingletaryYueAmes} while alternating with controller synthesis. The current paper differs from these works in that both learning and control are done simultaneously after every new state measurement, while providing rigorous theoretical bounds that the system will be safe with high probability at all times.

Other approaches have also been proposed in the area of safe reinforcement learning \cite{GarciaFernandez}, where control signals/actions are learnt using reinforcement learning techniques, but modified to ensure that the system is safe \cite{Junges,YuYangKolarWang,ZhangBastaniKumar}. These modifications include constraining exploration to strategies which are safe \cite{Junges}, solving a constrained Markov decision process where the constraints provide safety \cite{YuYangKolarWang}, and the use of shielding \cite{ZhangBastaniKumar} with a backup controller  which is guaranteed to be safe.  The algorithms employed usually learn the actions in a ``model-free'' manner, even if some knowledge of the system may be available (as is often the case in control), and as such the performance may be conservative. Furthermore, without a model, safety often can not be guaranteed during learning but only after a sufficiently long learning period  \cite{GarciaFernandez,ChengOroszMurrayBurdick,YuYangKolarWang}. The approach of \cite{Junges} also requires a finite state space.

The closest study to this paper is the work of~\cite{DevonportYinArcak}, where the authors use the results of~\cite{srinivas2009gaussian} to evaluate their confidence in learnt Gaussian processes modelling the system and the environment. However, \cite{DevonportYinArcak} do not provide computationally friendly methods for ensuring safety as their framework relies on Lyapunov functions that can be computationally difficult to find~\cite{parrilo2000structured}.

\paragraph*{Contributions}
In this paper, we propose a learning based control methodology that guarantees safety while both learning and control are carried out simultaneously. In our approach, we consider a nonlinear control-affine model subject to a process noise with unknown parameters. Safety of the agent is characterised by time-varying state constraints which at each time step bars the agent from parts of its state space. 

We start by considering a zero-mean Gaussian noise uncertainty with an unknown covariance matrix, which is learnt online from state measurements. 
In this case, we use the empirical covariance to construct a robust optimization problem for minimally modifying control actions generated by controllers (e.g., PID, model predictive, or reinforcement learning control) to ensure safety. Theorem~\ref{thm:safety_prob} shows that with  probability greater than $1-\delta$, the modified control action results in a safe state in the next time step if a certain optimization problem with an infinite number of constraints is feasible. Using robust optimisation techniques as outlined in Lemma~\ref{lemma:robust}, we reformulate the desired optimization as a problem subject to finitely many constraints that are ``tightened'' versions of the original safety constraints. The tightening is bigger for more stringent safety guarantee requirements (smaller $\delta$), but shrinks with time as more data becomes available for covariance estimation.
The proof of Theorem~\ref{thm:safety_prob} uses  Markov's concentration inequality, which can be conservative. Therefore, we consider an alternative safety bound using Chebyshev’s inequality in Theorem~\ref{thm:safety_prob_chebyshev}. However, the robust optimization problem stemming from the Chebyshev’s inequality is not convex. We show it can be cast as a convex problem in Theorem~\ref{thm:exact_relax}. This alternative formulation is proved to be less conservative for small $\delta$ and/or sufficiently large time.  The presented approaches in Theorem~\ref{thm:safety_prob} and Theorem~\ref{thm:safety_prob_chebyshev} rely on a minimum number of samples to ensure invertibility of the empirical covariance matrix, which is not possible at early time steps (before we gather enough measurements). We relax this assumption, at the expense of some conservatism and an extra assumption of the existence of an upper bound on the covariance of the process noise, in Theorem~\ref{thm:safety_prob_alternative}. 
We end this section by investigating the effect of noisy state measurements on safe learning for linear time-invariant dynamical systems. The safety guarantee in this case is given in Theorem~\ref{thm:safety_prob_noisy_measurement}.

We proceed by extending the problem formulation to admit non-zero-mean Gaussian process noise. In this case, both the mean and covariance of the noise are learnt from the state measurements. 
In this case, Theorem~\ref{thm:safety_prob_nonzero_mean} (counterpart to Theorem~\ref{thm:safety_prob}), Theorem~\ref{tho:optim_prob_chebyshev_nonzero_mean} (counterpart to Theorem~\ref{thm:safety_prob_chebyshev}), and Theorem~\ref{thm:safety_prob_nonzero_mean_alternative} (counterpart to Theorem~\ref{thm:safety_prob_alternative}) prove that the control signal obtained from the corresponding robust optimization problems (if feasible) can guarantee safety of the next visited state with probability greater than $1-\delta$. In all these cases, we can use {Lemma~\ref{lemma:robust}} to recast the robust optimization problem with infinitely-many constraints as one with a single tightened safety constraint.

Noting that most of the presented results are based on general concentration inequalities, we extend the results to non-Gaussian additive process noises in Section \ref{sec:nongaussian}. Theorem~\ref{thm:non_Gaussian} considers safety in the presence of additive zero-mean potentially non-Gaussian process noise, 
while Theorem~\ref{thm:non_Gaussian:non-zero-mean} considers the case with possibly non-zero-mean general process noise. 

Finally, we address the challenging, yet important, scenario where the process noise does not have fixed parameters but is state dependent. This scenario is relevant for robotic applications where the environment is not homogeneous, or control scenarios where the un-modelled dynamics of the agent is different based on the region of the state space being explored. 
In this case, we partition the state space and approximate the mean and covariance by piece-wise constant functions over the partitions. We then proceed to estimate these parameters for each partition based on state data obtained in that area. Assuming some regularity, i.e., the mean and covariance do not change abruptly across neighbouring partitions, we derive alternative bounds that can be potentially tighter than just using the data from each region. Finally, we present a recipe for merging the regions to reduce the complexity of the presented model. We provide guarantees on the quality of the estimated mean and covariance for the merged regions in Proposition~\ref{prop:merge}.         

\paragraph*{Notations}
For matrices $X$ and $Y$, we say that $X\succeq 0$ if $X$ is positive semi-definite, and that $X \succeq Y$ if $X-Y$ is positive semi-definite. Given a matrix $X$, a vector $x$, and a superscript $i$, we will denote $X^i$ as the $i$-th row of $X$ and $x^i$ as the $i$-th component of $x$. The Kronecker product is denoted by $\otimes$. The Euclidean and Frobenius norms of vectors and matrices are denoted by $\|\cdot\|_2$ and $\|\cdot\|_F$, respectively.

\section{Safe Learning in the Presence of  Uncertainties}
\label{sec:safe_learning_zero_mean}
We consider a discrete time nonlinear control-affine system
\begin{align}
\label{eqn:control_affine_model}
    x_{t+1} = f(x_t) + g(x_t) u_t + w_t,
\end{align}
where $x_t\in\mathbb{R}^n$ is the state, $u_t\in\mathbb{R}^m$ is the control input, and $w_t\in\mathbb{R}^n$ is a sequence of independent and identically distributed  random vectors with mean $\mu_w$ and covariance $\Sigma_w\succeq 0$. 
Assume that the model dynamics $f(\cdot)$ and $g(\cdot)$ are known. 
This setup, i.e., knowing the system dynamics while not knowing the process noise, can be motivated by that we are using an accurately-modelled agent in an unknown environment (e.g., search and rescue mission by expensive robots). Alternatively, when dealing with a not-fully-known agent, we can lump all model uncertainties (e.g., un-modelled dynamics, model mismatch, and parameter errors) into the process noise and attempt to learn the characteristics of the process noise while maneuvering our agent.

We have safety constraints of the form 
\begin{align} \label{eqn:state_constraint}
    H_tx_t\leq h_t,
\end{align}
where the inequality is element-wise. This means that, at each iteration, parts of the state space are off limits to us. This could be to avoid collision with other agents, or getting captured by the enemy. Not knowing $\mu_w$ and $\Sigma_w$ hinders our ability to satisfy this constraint, as we do not know how the process noise influences the state in the next time step after implementing the control action. Therefore, we use the state measurements of the system to learn these parameters and guarantee safety. In the next section, we consider the case where the uncertainty is Gaussian. Initially the case where $\mu_w=0$ and the covariance matrix $\Sigma_w$ is unknown is studied. Then, the analysis is extended to the case where the mean is unknown and potentially non-zero. In Section~\ref{sec:nongaussian}, the Gaussianity assumption is relaxed and the safe learning problem is tackled for zero mean and non-zero mean cases. In Section~\ref{sec:piecewise_constant} we address the problem of spatially varying Gaussian uncertainties.   

\section{Gaussian Uncertainties} 
In this section we study the problem where the uncertainties are Gaussian. First, we address the problem where the mean of the disturbances is zero. Later, we consider the case where the mean is non-zero and needs to be learned in addition to the covariance matrix.
\subsection{Zero Mean Gaussian Uncertainties}
\label{sec:zero_mean}
The empirical estimate of the covariance after $t\geq 1$ time steps is given by
\begin{align*}
    \widehat{\Sigma}_w:=\frac{1}{t} \sum_{k=0}^{t-1} w_kw_k^\top  &=\frac{1}{t} \sum_{k=0}^{t-1} (x_{k+1}- \widehat{x}_{k+1})(x_{k+1} - \widehat{x}_{k+1})^\top,
\end{align*}
where $\widehat{x}_{k+1} = f(x_k) + g(x_k)u_k$.
Note that $\widehat{\Sigma}_w$ has Wishart distribution $\mathcal{W}_n(t^{-1}\Sigma_w,t)$,
and that $\widehat{\Sigma}_w$ is an unbiased estimate of $\Sigma_w$, i.e., 
$
    \mathbb{E}\{\widehat{\Sigma}_w\}
    =({1}/{t}) \sum_{k=0}^{t-1} \mathbb{E}\{w_kw_k^\top\}=\Sigma_w.$

Our approach to controller design in this paper is to take an existing nominal\footnote{The nominal control can, e.g., be generated by standard control design techniques (such as PID or model predictive control) or by a reinforcement learning algorithm such as \cite{Lillicrap_DDPG} (neither of which takes into account safety).} control input  $\bar{u}_t$, and modify this control  slightly at each iteration to ensure safety. We do this by considering  the following optimization problem with safety constraints which is to be solved at each  $t > n+1 $:
\begin{subequations} \label{optim_prob_zero_mean}
\begin{align}
    u_t=&\argmin_{u}  \quad \mathbf{d}(u,\bar{u}_t),\\
    &\quad\, \mathrm{s.t.}  \quad H_{t+1}(f(x_t) + g(x_t)u +w)\leq h_{t+1}, \quad \\
    & \quad \quad \quad \quad \forall w: w^\top \widehat{\Sigma}_w^{-1} w \leq \frac{t n}{(t-n-1) \delta}, \label{eqn:w_condition_zero_mean}
\end{align}
\end{subequations}
where $\mathbf{d}(\cdot,\cdot)$ denotes a distance between its two arguments. For example, one can choose $\mathbf{d}(u,\bar{u})$ to be $\|u-\bar{u}\|_2$.

We have the following result on high probability safety guarantees of the learned-controlled system.
\begin{theorem}
\label{thm:safety_prob}
Assume that problem~\eqref{optim_prob_zero_mean} is feasible. Then, if we implement the control action $u_t$, $x_{t+1}$ is safe with probability of at least $1-\delta$, i.e., $\mathbb{P}\{H_{t+1} x_{t+1}\leq h_{t+1}\}\geq 1-\delta$.
\end{theorem}

\begin{proof} 
We have 
\begin{align}
     \mathbb{P}\{H_{t+1} x_{t+1}\leq h_{t+1}\} & \geq \mathbb{P}\left\{
     w^\top \widehat{\Sigma}_w^{-1} w \leq \frac{t n}{(t-n-1)\delta} 
     \right\}\notag \\
     & \geq 1-\frac{\mathbb{E}\{w^\top \widehat{\Sigma}_w^{-1} w\}}{tn/((t-n-1)\delta)}, \label{eqn:expectation1_0}
\end{align} 
where the last inequality follows from an application of Markov's inequality for scalar random variables~\cite[\S\,2.1]{boucheron2013concentration}. We can compute
\begin{subequations}
\begin{align}
        \mathbb{E}\{w^\top \widehat{\Sigma}_w^{-1} w\} & = \trace \mathbb{E}\{w w^\top \widehat{\Sigma}_w^{-1} \} \notag\\
& = \trace \left( \mathbb{E}\{w w^\top\}\label{eqn:expectation1} \mathbb{E}\{\widehat{\Sigma}_w^{-1} \} \right) \\
& = \trace \left(\Sigma_w \times \frac{t}{t-n-1} \Sigma_w^{-1} \right) \label{eqn:expectation1_2}\\
& =  \frac{t}{t-n-1} \trace(I)  = \frac{t n}{t-n-1}.\label{eqn:expectation1_3}
\end{align}
\end{subequations}
The equality in \eqref{eqn:expectation1} follows from independence of $w$ and $ \widehat{\Sigma}_w^{-1}$, since in problem \eqref{optim_prob_zero_mean} $w$ refers to $w_t$, while $ \widehat{\Sigma}_w^{-1}$ is a function of $(w_0,\dots,w_{t-1})$. The equality in \eqref{eqn:expectation1_2} is the consequence of results on the expectation of an inverted Wishart distribution, see e.g.  \cite[Lemma 7.7.2]{anderson2003introduction}.
The result $ \mathbb{P}\{H_{t+1} x_{t+1}\leq h_{t+1}\} \geq 1 - \delta$ follows from substituting~\eqref{eqn:expectation1_3} into~\eqref{eqn:expectation1_0}.
\end{proof}

Although in principle problem \eqref{optim_prob_zero_mean} can be solved for any $t$, the safety guarantees provided by Theorem \ref{thm:safety_prob} are valid only for $t > n+1$, as we need to wait until $\widehat{\Sigma}_w$ becomes invertible. 

Note  that problem~\eqref{optim_prob_zero_mean} as written involves infinitely many constraints and is thus not easy to solve numerically. We next provide a computationally-friendly reformulation of~\eqref{optim_prob_zero_mean}.  To do so, we first present the following lemma. 

\begin{lemma} \label{lemma:robust} 
For $W\succeq 0$ and $d\geq 0$, 
$
    \{u\,|\,a^\top u+b^\top w \leq c,\forall w: w^\top W w\leq d \} =\{u\,|\,a^\top u\leq c-\sqrt{d} \|W^{-1/2}b\|_2\}
$.
\end{lemma}

\begin{proof}
With the change of variables $\bar{w}=W^{1/2}w$ and $\bar{b}=W^{-1/2} b$, we have 
$
    \{u\,|\,a^\top u+b^\top w \leq c,\forall w: w^\top W w\leq d \}
    =
    \{u\,|\,a^\top u+\bar{b}^\top\bar{w} \leq c,\forall \bar{w}: \bar{w}^\top \bar{w}\leq d \}.
$
Then, following the approach of~\cite[Example~1.3.3]{ben2009robust}, we can obtain $\{u\,|\,a^\top u+\bar{b}^\top\bar{w} \leq c,\forall \bar{w}: \bar{w}^\top \bar{w}\leq d \}=\{u\,|\,\sqrt{d}\|\bar{b}\|_2 \leq c-a^\top u \}$. 
\end{proof}

Using Lemma \ref{lemma:robust} on each component of $H_{t+1}(f(x_t) + g(x_t)u +w)\leq h_{t+1}$ (with the identifications $u \mapsto u$, $w \mapsto w$, $W \mapsto \widehat{\Sigma}_w^{-1}$, $d \mapsto tn/((t-n-1)\delta)$, $a^\top \mapsto (H_{t+1} g(x_t))^i$, $b^\top \mapsto H_{t+1}^i$, $c \mapsto (h_{t+1} - H_{t+1} f(x_t))^i$), problem \eqref{optim_prob_zero_mean} can be reformulated as the following:
\begin{subequations} \label{optim_prob_reformulation_zero_mean}
\begin{align}
    u_t=&\argmin_{u}  \quad \mathbf{d}(u,\bar{u}_t),\\
    &\quad \,\mathrm{s.t.}\quad   H_{t+1}(f(x_t) + g(x_t)u)\leq h_{t+1}-e_{t+1}, 
\end{align}
\end{subequations}
where $e_{t+1}$ is a vector whose $i$-th component is equal to
\begin{equation}
\label{eqn:e_vector_zero_mean}
    e_{t+1}^i =\sqrt{\frac{tn}{(t-n-1)\delta}}\left\| \widehat{\Sigma}_w^{1/2} (H_{t+1}^i)^\top \right\|_2.
\end{equation}
This reformulation \eqref{optim_prob_reformulation_zero_mean} has an interpretation which amounts to tightening of the constraints for a system without uncertainties. The tightening is larger at the beginning and becomes smaller as we gather  more information, i.e., our confidence in the learned mean and covariance increases.

\begin{remark}[Multi-Agent Safe Learning of the Environment] The problem formulation of this paper investigates the motion of a single robot in an environment. However, in practice, there are often more than one agent. In that case, we need to decouple the robots, while ensuring that they do not collide and can accomplish tasks collaboratively. For this, we can use the methodology of~\cite{wood2020collision} to construct time-varying constraints of the form~\eqref{eqn:state_constraint}, that can ensure collision avoidance and guarantee collaboration for task assignment. Then, we can use the methodology of this paper for safely maneuvering each robot in the environment. 
\end{remark}

\begin{remark}[Safety Guarantee versus Feasibility] For a given $\delta$, the optimization problem in~\eqref{optim_prob_zero_mean}, which is equivalent to~\eqref{optim_prob_reformulation_zero_mean}, may be infeasible. In this case, we cannot ensure the safety of the next visited state with probability of at least $1-\delta$. We can increase $\delta$ to get a feasible problem (if one exists). A bisection algorithm can be used to find the minimum probability of unsafe manoeuvre $\delta$ while being feasible.
\end{remark}

\subsubsection{Safety bounds using Chebyshev's inequality}
In Theorem~\ref{thm:safety_prob}, we have made use of Markov's inequality in deriving our high probability safety bounds. In this section, we consider alternative safety bounds making use of Chebyshev's inequality \cite[\S\,2.1]{boucheron2013concentration}, which is known to provide sharper bounds than Markov's inequality in many situations. We begin by providing a preliminary result.
\begin{lemma}
\label{lemma:variance_term}
Let $v\sim \mathcal{N}(0, s_1 \Sigma)$ and $\widehat{\Sigma} \sim \mathcal{W}_n(s_2 \Sigma, \tau)$ be independent. Then for $\tau > n+3$,
$$\mathbb{E}\{(v^\top \widehat{\Sigma}^{-1} v)^2\} = \frac{s_1^2}{s_2^2} \frac{n(n+2)}{(\tau-n-1)(\tau-n-3)}.$$
\end{lemma}

\begin{proof}
See Appendix \ref{appendix:variance_proof}.
\end{proof}

Now consider the following optimization problem:
\begin{subequations} \label{optim_prob_chebyshev_zero_mean}
\begin{align}
    &u_t=\argmin_{u}  \quad \mathbf{d}(u,\bar{u}_t),\\
    &\mathrm{s.t.} \quad \;\;  H_{t+1}(f(x_t) + g(x_t)u +w)\leq h_{t+1}, \\
    & \quad \quad \quad \quad \forall w:
     \Big|w^\top \widehat{\Sigma}_w^{-1} w - \frac{tn}{t-n-1} \Big| \notag\\ 
    & \quad \quad \quad \quad \quad \quad \leq \sqrt{\frac{2  t^2(t-1)n}{(t-n-1)^2(t-n-3)\delta} }, \label{eqn:w_condition_cheybshev_zero_mean}
\end{align}
\end{subequations}
which is to be solved at each $t > n+3$. \begin{theorem}
\label{thm:safety_prob_chebyshev}
Assume that problem~\eqref{optim_prob_chebyshev_zero_mean} is feasible. Then, if we implement the control action $u_t$, $x_{t+1}$ is safe with probability of at least $1-\delta$.
\end{theorem}

\begin{proof}
We have 
\begin{align*}
     &\mathbb{P}\{H_{t+1} x_{t+1}\leq h_{t+1}\} \\ & \geq \mathbb{P}\Bigg\{
     \! \bigg|w^\top \widehat{\Sigma}_w^{-1} w \!-\! \frac{tn}{t-n-1} \bigg| \! \leq\! \sqrt{\!\frac{2  t^2(t-1)n}{(t-n-1)^2(t-n-3)\delta} }
     \Bigg\} \\
     & \geq 1-\frac{(t-n-1)^2(t-n-3)\delta}{2  t^2(t-1)n} \textnormal{Var}\{w^\top \widehat{\Sigma}_w^{-1} w\}
\end{align*} 
where the last inequality follows from \eqref{eqn:expectation1_3} and an application of Chebyshev's inequality. As $w \sim \mathcal{N}(0,\Sigma_w)$ and $\widehat{\Sigma}_w \sim \mathcal{W}_n(t^{-1} \Sigma_w,t)$, Lemma \ref{lemma:variance_term} and \eqref{eqn:expectation1_3} yield
\begin{align*}
 \textnormal{Var}\{w^\top \widehat{\Sigma}_w^{-1} w\} &= \mathbb{E}\{(w^\top \widehat{\Sigma}_w^{-1} w)^2\} - \left(\mathbb{E}\{w^\top \widehat{\Sigma}_w^{-1} w\}\right)^2\\
 &= \frac{t^2 n(n+2)}{(t-n-1)(t-n-3)} - \left(\frac{tn}{t-n-1} \right)^2\\
 &= \frac{2  t^2(t-1)n}{(t-n-1)^2(t-n-3)}.
\end{align*}
Hence $\mathbb{P}\{H_{t+1} x_{t+1}\leq h_{t+1}\}\geq 1-\delta$.
\end{proof}
The condition $t > n+3$ is needed in order for the second order moments of $\widehat{\Sigma}^{-1}$ to exist \cite[Theorem 3.2]{Haff79}.

For notational convenience, let us now denote
\begin{subequations}
\label{eqn:E_V}
\begin{align}
E & := \frac{t n}{t-n-1} \\
V & :=  \frac{2  t^2(t-1)n}{(t-n-1)^2(t-n-3)}
\end{align}
\end{subequations}
for the mean and variance of $w^\top \widehat{\Sigma}_w^{-1} w$.
Condition \eqref{eqn:w_condition_cheybshev_zero_mean} is equivalent to 
\begin{equation}
\label{eqn:w_condition_equiv_cheybshev_zero_mean}
     \forall w: E - \sqrt{V/\delta} \leq w^\top \widehat{\Sigma}_w^{-1} w \leq E + \sqrt{V/\delta}.
\end{equation} 
This constraint captures the region in between two ellipsoids, which in general is non-convex. This can make it difficult to find a computationally efficient exact reformulation of problem~ \eqref{optim_prob_chebyshev_zero_mean}. We consider an alternative optimization problem which ignores the first inequality in \eqref{eqn:w_condition_equiv_cheybshev_zero_mean}:
\begin{subequations} \label{optim_prob2_chebyshev_zero_mean}
\begin{align}
    u_t=&\argmin_{u}  \quad \mathbf{d}(u,\bar{u}_t),\\
    &\quad\, \mathrm{s.t.} \quad H_{t+1}(f(x_t) + g(x_t)u +w)\leq h_{t+1}, \\
    & \quad \quad \quad \quad \forall w:
     w^\top \widehat{\Sigma}_w^{-1} w  \leq E +  \sqrt{V/\delta}, \label{eqn:w_condition2_cheybshev_zero_mean}
\end{align}
\end{subequations}
which is again to be solved at each $t > n+3$. Problem \eqref{optim_prob2_chebyshev_zero_mean} has the same structure as problem \eqref{optim_prob_zero_mean}, and can also be reformulated into the computationally efficient form \eqref{optim_prob_reformulation_zero_mean}, with 
 \begin{equation}
 \label{eqn:e_vector_chebyshev_zero_mean}
  e_{t+1}^i =\sqrt{E + \sqrt{V/\delta}}\left\| \widehat{\Sigma}_w^{1/2} (H_{t+1}^i)^\top \right\|_2. \end{equation}
Interestingly, we can prove that this relaxation is exact. 

\begin{theorem}
\label{thm:exact_relax}
 Problems~\eqref{optim_prob_chebyshev_zero_mean} and~\eqref{optim_prob2_chebyshev_zero_mean} are equivalent. 
\end{theorem}

\begin{proof} Define $B:=H_{t+1}g(x_t)$, $b:=h_{t+1}-H_{t+1} f(x_t)$, $c:=E - \sqrt{V/\delta}$, $d:=E +  \sqrt{V/\delta}$, 
$
    \mathcal{U}_1:=\{u\,|\,Bu+H_{t+1}w\leq b,\forall w:c\leq w^\top \widehat{\Sigma}_w^{-1}w\leq d \}$, and 
    $\mathcal{U}_2:=\{u\,|\,Bu+H_{t+1}w\leq b,\forall w: w^\top \widehat{\Sigma}_w^{-1}w\leq d \}.$
Note that~\eqref{optim_prob_chebyshev_zero_mean} is $\argmin_{u\in \mathcal{U}_1}\mathbf{d}(u,\bar{u}_t)$ and~\eqref{optim_prob2_chebyshev_zero_mean} is $\argmin_{u\in \mathcal{U}_2} \mathbf{d}(u,\bar{u}_t)$. To prove this proposition, we must show that $\mathcal{U}_1=\mathcal{U}_2$. It is easy to see that $\mathcal{U}_2\subseteq\mathcal{U}_1$. Therefore, we must show that $\mathcal{U}_1\subseteq \mathcal{U}_2$. The proof is by contradiction.
Assume that $\mathcal{U}_1\nsubseteq \mathcal{U}_2$. Then, there must exist a $u\in\mathcal{U}_1$ such that $u\notin\mathcal{U}_2$. Therefore, there must exist a $w$ such that $w^\top \widehat{\Sigma}_w^{-1}w \leq d$ and $B^iu+H_{t+1}^iw> b^i$ for some $i$. 

Assume first that $w\neq 0$. Construct 
\begin{align*}
    w':=\mathrm{sign}(H_{t+1}^iw)\sqrt{\frac{d}{w^\top \widehat{\Sigma}_w^{-1}w}}w.
\end{align*}
We can see that $w^{\prime\top}\widehat{\Sigma}_w^{-1}w'=d$. In particular, this implies $c\leq w^{\prime\top} \widehat{\Sigma}_w^{-1} w^\prime\leq d$. Furthermore, 
\begin{align*}
    B^iu+H_{t+1}^iw'
    =&B^iu+\sqrt{\frac{d}{w^\top \widehat{\Sigma}_w^{-1}w}}|H_{t+1}^iw|\\
    \geq & B^i u+|H_{t+1}^iw|
    \geq  B^i u+H_{t+1}^iw
    >b^i,
\end{align*}
where the first inequality follows from  $w^\top \widehat{\Sigma}_w^{-1}w\leq d$.
This is in contradiction with $u\in\mathcal{U}_1$. 

Now, we treat the case where $w=0$. Therefore, there exists some $i$ such that $B^iu> b^i$. Since $u\in\mathcal{U}_1$, $B^iu+H_{t+1}^i\bar{w}\leq b^i$ for any $\bar{w}$ such that $c\leq \bar{w}^\top \widehat{\Sigma}_w^{-1}\bar{w}\leq d$. Pick one such $\bar{w}$. As $B^iu> b^i$, it must be that $H_{t+1}^i\bar{w}<0$. Define $\bar{w}^\prime=-\bar{w}$. By construction $H_{t+1}^i\bar{w}^\prime=-H_{t+1}^i\bar{w}>0$ and $c\leq \bar{w}^{\prime\top} \widehat{\Sigma}_w^{-1}\bar{w}^\prime\leq d$. This implies that $B^iu+H_{t+1}^i\bar{w}^\prime>B^iu>b^i$, which is in contradiction with $u\in\mathcal{U}_1$.  
\end{proof}

\emph{Comparison with problem \eqref{optim_prob_zero_mean}}: Next, we will make some comparisons between problem \eqref{optim_prob_zero_mean}, which uses Markov's inequality to derive its safety bounds, and problem \eqref{optim_prob2_chebyshev_zero_mean}, which is based on Chebyshev's inequality (recall that problems \eqref{optim_prob_chebyshev_zero_mean} and \eqref{optim_prob2_chebyshev_zero_mean} are equivalent by Theorem \ref{thm:exact_relax}). 

First, we note that the safety guarantees for problem \eqref{optim_prob_zero_mean} only apply for $t > n+1$, while the safety guarantees for problem \eqref{optim_prob2_chebyshev_zero_mean} only apply for $t > n+3$. So for $t=n+2$ or $t= n+3$, problem \eqref{optim_prob2_chebyshev_zero_mean} should not be used if one wants safety guarantees. Next, we provide some conditions on when problem \eqref{optim_prob2_chebyshev_zero_mean} is less conservative than problem \eqref{optim_prob_zero_mean}.  

\begin{lemma}
\label{lemma:chebyshev_markov_comparison}
Problem \eqref{optim_prob2_chebyshev_zero_mean} is less conservative than problem~\eqref{optim_prob_zero_mean} if 
\begin{equation}
    \label{eqn:chebyshev_better_markov_condition}
    \frac{2(t-1)}{t-n-3} < \frac{n(1-\delta)^2}{\delta}.
\end{equation}
In particular,  \eqref{optim_prob2_chebyshev_zero_mean} is less conservative than \eqref{optim_prob_zero_mean}: \\(i) for all $t > n+3$ if
    $\delta < ({2n + 3 - \sqrt{3(n+1)(n+3)}})/{n}$; \\(ii) for sufficiently large $t$ if
   $\delta < ({n+1 - \sqrt{2n+1}})/{n} $.
\end{lemma}

\begin{proof}
Comparing \eqref{eqn:w_condition_zero_mean} and \eqref{eqn:w_condition2_cheybshev_zero_mean}, and using the expressions for $E$ and $V$ in \eqref{eqn:E_V}, we see that  \eqref{optim_prob2_chebyshev_zero_mean} is less conservative than \eqref{optim_prob_zero_mean} when 
$ E + \sqrt{V/\delta} < {E}/{\delta} $ which is equivalent to
$ {V}/{\delta} < {E^2 (1-\delta)^2}/{\delta^2} $, which in turn is equivalent to  ${2(t-1)}/({t-n-3}) < {n(1-\delta)^2}/{\delta}$,
which proves \eqref{eqn:chebyshev_better_markov_condition}.

As $(t-1)/(t-n-3)$ is a decreasing function of $t$, for condition \eqref{eqn:chebyshev_better_markov_condition} to hold for all $t > n+3$, we need it to hold for $t=n+4$, i.e., 
$ 2(n+3) < {n(1-\delta)^2}/{\delta}.$ 
This then leads to the condition in~(i).

For \eqref{eqn:chebyshev_better_markov_condition} to hold for large $t$, we take $t \rightarrow \infty$ to obtain 
$ 2 < {n(1-\delta)^2}/{\delta}.$ 
This leads to the condition in (ii).
\end{proof} 
In the case of unicycle robots with $n=3$, as in our numerical example in Section \ref{sec:numerical}, Lemma \ref{lemma:chebyshev_markov_comparison} says that problem \eqref{optim_prob2_chebyshev_zero_mean} will always be less conservative than problem~\eqref{optim_prob_zero_mean} if $\delta < (9-\sqrt{72})/3 \approx 0.172$, while problem \eqref{optim_prob2_chebyshev_zero_mean} is less conservative for large $t$ when $\delta < (4-\sqrt{7})/3 \approx 0.451$.

\subsubsection{Safety bounds valid for all $t$}
As previously mentioned, the safety guarantees provided by solving problem \eqref{optim_prob_zero_mean} are valid only for $t > n+1$, while the safety guarantees for problem \eqref{optim_prob2_chebyshev_zero_mean} require $t> n+3$. 
Here we present alternative safety bounds which are valid for all $t \geq 1$, although they are more conservative than problems  \eqref{optim_prob_zero_mean} and \eqref{optim_prob2_chebyshev_zero_mean} for larger values of $t$. As such, these bounds are most useful for small $t$. The techniques used here can also be extended to handle non-Gaussian uncertainties, see Section \ref{sec:nongaussian}.

We will assume that an upper bound for the covariance of the process
noise is known.
\begin{assumption}
\label{assumption:Sigma_w}
A positive scalar $\sigma$ exists such that $\Sigma_w\preceq \sigma I$.
\end{assumption}


\begin{lemma} \label{lemma:alternative} Under Assumption~\ref{assumption:Sigma_w}, for all $t\geq 1$, 
\begin{align}
    \mathbb{P}\left\{\Sigma_w\preceq \widehat{\Sigma}_w+ \sigma \sqrt{\frac{n(n+1)}{\delta t }} I\right\}\geq 1-\delta. \label{lemma_bound}
\end{align}
\end{lemma}

\begin{proof} 
Recall that $\widehat{\Sigma}_w \sim \mathcal{W}_n(t^{-1}\Sigma_w,t)$ and $\mathbb{E}\{\widehat{\Sigma}_w\}=\Sigma_w$. As a result,
\begin{align*}
    \mathbb{E}\{(\Sigma_w-\widehat{\Sigma}_w)^2\} & = \mathbb{E}\{\hat{\Sigma}_w^2 \} - \Sigma_w^2 \\
    & = \frac{\trace(\Sigma_w)}{t} \Sigma_w + \frac{t+1}{t} \Sigma_w^2 - \Sigma_w^2 \\
    & = \frac{\trace(\Sigma_w)}{t} \Sigma_w + \frac{1}{t} \Sigma_w^2
\end{align*}
where the expression for $\mathbb{E}\{\hat{\Sigma}_w^2 \}$ can be found in, e.g. \cite[Theorem 3.1]{Haff79}.
Now, the matrix version of Chebyshev's inequality given in Lemma~\ref{lemma:Chebyshev} of Appendix~\ref{appendix:1} results in
\begin{align*}
    \mathbb{P}\{\Sigma_w-\widehat{\Sigma}_w\preceq \varepsilon I\}
    &\geq 1-\trace(\mathbb{E}\{ (\Sigma_w-\widehat{\Sigma}_w)^{2}\varepsilon^{-2}\})\\
    &=  1-\frac{(\trace(\Sigma_w))^2+\trace(\Sigma_w^2)}{\varepsilon^2 t}.
\end{align*}
Thus, by Assumption \ref{assumption:Sigma_w}:
$
    \mathbb{P}\{\Sigma_w\!-\!\widehat{\Sigma}_w\!\preceq\! \varepsilon I\}
    \!\geq\! 1\!-\!{\sigma^2n(n\!+\!1)}/({\varepsilon^2t}).
$
Setting $ \varepsilon^2 \!=\! {\sigma^2 n(n\!+\!1)}/({\delta t})$ gives (\ref{lemma_bound}). 
\end{proof}

Now, we can adapt~\eqref{optim_prob_zero_mean} to:
\begin{subequations} \label{eqn:control_approximation}
\begin{align}
    &u_t= \argmin_{u}  \quad \mathbf{d}(u,\bar{u}_t),\\
    & \mathrm{s.t.}   \quad H_{t+1}(f(x_t) + g(x_t)u +w)\leq h_{t+1}, \nonumber\\ 
    &\quad \quad \;\forall w:
     w^\top \left(\widehat{\Sigma}_w+ \sigma \sqrt{\frac{2 n(n+1)}{\delta t }} I\right)^{-1} w \leq \frac{2n}{\delta}.
\end{align}
\end{subequations}
\begin{theorem}
\label{thm:safety_prob_alternative}
Assume that problem~\eqref{eqn:control_approximation} is feasible and Assumption~\ref{assumption:Sigma_w} holds. Then, if we implement the control action $u_t$, $x_{t+1}$ is safe with probability of at least $1-\delta$.
\end{theorem}

\begin{proof} Lemma~\ref{lemma:alternative} implies that 
\begin{align*}
     \mathbb{P}\left\{\Sigma_w\preceq \widehat{\Sigma}_w+\sigma \sqrt{\frac{2 n(n+1)}{\delta t }} I\right\}\geq 1-\frac{\delta}{2}.
\end{align*}
Let $A$ denote the event that $\Sigma_w\preceq \widehat{\Sigma}_w+\sigma \sqrt{{2 n(n+1)}/({\delta t })} I$. We have
\begin{align*}
&\mathbb{P}\left\{w^\top \left(\widehat{\Sigma}_w+\sigma \sqrt{\frac{2 n(n+1)}{\delta t }} I\right)^{-1} w\leq \frac{2n}{\delta} \Bigg| A \right\}
\\ &\geq 
    \mathbb{P}\left\{w^\top \Sigma_w^{-1} w\leq \frac{2n}{\delta}\right\} \\
    &\geq 
    1-\frac{\mathbb{E}\{w^\top \Sigma_w^{-1} w\}}{2n/\delta}
    =1-\frac{\trace(\mathbb{E}\{ w w^\top \Sigma_w^{-1} \})}{2n/\delta}
    =1-\frac{\delta}{2},
\end{align*}
where the last inequality follows from an application of Markov's inequality (for scalar random variables).
Therefore, 
\begin{align*}
    & \mathbb{P}\{H_{t+1} x_{t+1}\leq h_{t+1}\} 
     \\ & \geq \mathbb{P}\left\{w^\top\! \left(\widehat{\Sigma}_w+\sigma \sqrt{\frac{2 n(n+1)}{\delta t }} I\right)^{-1}\! w\leq \frac{2n}{\delta} \right\} \\
     & \geq \mathbb{P}\left\{w^\top\! \left(\widehat{\Sigma}_w+\sigma \sqrt{\frac{2 n(n+1)}{\delta t }} I\right)^{-1} \hspace{-.1in}w\leq \frac{2n}{\delta} \Bigg| A \right\} \mathbb{P}\{A\} 
     \\ & \geq \left(1-\frac{\delta}{2}\right)^2 \geq  1 - \delta,
\end{align*} 
where the second inequality follows from the law of total probability. 
This concludes the proof.
\end{proof}
By using Lemma \ref{lemma:robust}, problem \eqref{eqn:control_approximation} can also be reformulated into the computationally efficient form \eqref{optim_prob_reformulation_zero_mean}, with 
 \begin{equation}
 \label{eqn:e_vector_alternative_zero_mean}
  e_{t+1}^i =\sqrt{\frac{2n}{\delta}}\Bigg\| \Bigg(\widehat{\Sigma}_w + \sigma \sqrt{\frac{2n(n+1)}{\delta t}}I \Bigg)^{1/2}(H_{t+1}^i)^\top \Bigg\|_2. \end{equation}

\begin{remark} \label{remark:what_is_better} For $t\gg 1$, the optimization problem \eqref{eqn:control_approximation} provides a more conservative control action in comparison with~\eqref{optim_prob_zero_mean}, because  $\{w|w^\top (\widehat{\Sigma}_w+ \sigma \sqrt{{2 n(n+1)}/({\delta t })} I)^{-1}w\leq 2n/\delta\}\approx\{w|w^\top \widehat{\Sigma}_w^{-1}w\leq 2n/\delta\}$, $\{w|w^\top \widehat{\Sigma}_w^{-1} w \leq {t n}/({(t-n-1) \delta})\}\approx\{w|w^\top \widehat{\Sigma}_w^{-1}w\leq n/\delta\}$, and $\{w|w^\top \widehat{\Sigma}_w^{-1}w\leq n/\delta\}\subsetneq  \{w|w^\top \widehat{\Sigma}_w^{-1}w\leq 2n/\delta\}$. However, for $t\leq n+1$, the optimization problem in~\eqref{eqn:control_approximation} is more useful as we do not need to wait  until $\widehat{\Sigma}_w$ becomes invertible, which is a requirement for~\eqref{optim_prob_zero_mean}. In practice, we can use~\eqref{eqn:control_approximation} for small $t$ and switch to~\eqref{optim_prob_zero_mean} as $t$ gets larger.
\end{remark}

A more refined comparison is given in the following result.
\begin{lemma}
\label{lemma:less_conservative_alt}
Problem \eqref{optim_prob_zero_mean} is less conservative than \eqref{eqn:control_approximation} if
\begin{align}
& \frac{t}{t-n-1} H_{t+1}^i \widehat{\Sigma}_w (H_{t+1}^i)^\top \notag\\ & \quad  < 2 H_{t+1}^i \Bigg(\widehat{\Sigma}_w + \sigma \sqrt{\frac{2n(n+1)}{\delta t}}\Bigg) (H_{t+1}^i)^\top, \quad \forall i. \label{eqn:less_conservative_condition_alt} 
\end{align}
In particular, \eqref{optim_prob_zero_mean} is less conservative than \eqref{eqn:control_approximation} $\forall t > 2(n+1)$.

\end{lemma}

\begin{proof}
By looking at the reformulated versions, we see from \eqref{eqn:e_vector_zero_mean} and \eqref{eqn:e_vector_alternative_zero_mean} that problem \eqref{optim_prob_zero_mean} will be less conservative than problem \eqref{eqn:control_approximation} if
\begin{align*}
&\frac{tn}{(t-n-1)\delta}\left\| \widehat{\Sigma}_w^{1/2} (H_{t+1}^i)^\top \right\|^2_2 \\ 
& \quad < \frac{2n}{\delta}\Bigg\| \Bigg(\widehat{\Sigma}_w + \sigma \sqrt{\frac{2n(n+1)}{\delta t}}I \Bigg)^{1/2}(H_{t+1}^i)^\top \Bigg\|^2_2, \forall i,
\end{align*}
which is easily seen to be equivalent to \eqref{eqn:less_conservative_condition_alt}.

Furthermore, as 
$$2 H_{t+1}^i \widehat{\Sigma}_w  (H_{t+1}^i)^\top < 2 H_{t+1}^i \Bigg(\widehat{\Sigma}_w + \sigma \sqrt{\frac{2n(n+1)}{\delta t}}\Bigg) (H_{t+1}^i)^\top, $$
condition \eqref{eqn:less_conservative_condition_alt} will be satisfied if $t/(t-n-1) < 2$, or equivalently $ t > 2(n+1).$
\end{proof}

\begin{remark}
\label{remark:less_conservative_alt}
While \eqref{eqn:less_conservative_condition_alt} is a sufficient condition, in general it is not always clear which of the problems \eqref{optim_prob_zero_mean} and  \eqref{eqn:control_approximation} is less conservative. This is because one could  have situations where \eqref{eqn:less_conservative_condition_alt} is only satisfied for some components $i$,  with the inequality going in the opposite direction for the other components. 
\end{remark}

A similar condition to \eqref{eqn:less_conservative_condition_alt} can also be derived to test whether problem \eqref{optim_prob2_chebyshev_zero_mean} is less conservative than problem \eqref{eqn:control_approximation}.

\subsubsection{Adaptive selection of optimization problems}
We have presented a number of different optimization problems such as \eqref{optim_prob_zero_mean},  \eqref{optim_prob2_chebyshev_zero_mean}, and \eqref{eqn:control_approximation}, which have different ranges of validity and conservativeness. 
We note that it is possible to switch between these problems at different times, depending on whichever is less conservative. 
One possible procedure is summarized as Algorithm \ref{alg:safe_learning_zero_mean}. 
Note that Algorithm \ref{alg:safe_learning_zero_mean} as written only solves problem \eqref{eqn:control_approximation} for times $t \leq n+1$. By Lemma \ref{lemma:less_conservative_alt}, one could also compare problems  \eqref{eqn:control_approximation} and \eqref{optim_prob_zero_mean} at later times (e.g. until $t > 2(n+1)$), although as mentioned in Remark \ref{remark:less_conservative_alt} a clear-cut comparison may not always be straightforward. 

\begin{algorithm}[t]
\caption{Safe learning and control for zero mean Gaussian uncertainties}
\label{alg:safe_learning_zero_mean}
\begin{algorithmic}[1]
\For{$t = 1, 2, \dots$}
    \If{$t \leq n+1$}
        \State Solve problem \eqref{eqn:control_approximation} via \eqref{optim_prob_reformulation_zero_mean} and \eqref{eqn:e_vector_alternative_zero_mean}
    \ElsIf{$t=n+2$ or $t=n+3$}
        \State Solve problem \eqref{optim_prob_zero_mean} via \eqref{optim_prob_reformulation_zero_mean} and \eqref{eqn:e_vector_zero_mean}
    \Else
	    \If{ $2(t-1)/(t-n-3) < n(1-\delta)^2/\delta$ }
	        \State Solve problem \eqref{optim_prob2_chebyshev_zero_mean} via \eqref{optim_prob_reformulation_zero_mean} and \eqref{eqn:e_vector_chebyshev_zero_mean}
	     \Else
	        \State Solve problem \eqref{optim_prob_zero_mean} via \eqref{optim_prob_reformulation_zero_mean} and \eqref{eqn:e_vector_zero_mean}
	   \EndIf
	\EndIf
\EndFor
\end{algorithmic}
\end{algorithm}

\subsection{Noisy state measurements}
In this subsection, we investigate the effect of noisy statement measures of the form:
\begin{align*}
    y_k=x_k+v_k,
\end{align*}
where $v_k$ is a zero mean noise with potentially unknown variance $\Sigma_v$. In this case, we will restrict ourselves to linear models in \eqref{eqn:control_affine_model}.

\begin{assumption} \label{assumption:linear}
$f(x)=Ax$ and $g(x)=B$. 
\end{assumption}

For $t\geq 2$ and under Assumption~\ref{assumption:linear}, we can construct the empirical estimate of the covariance using the noisy measurements according to
\begin{align*}
    \widehat{\Sigma}_w:=
    \frac{1}{\lfloor \frac{t}{2}\rfloor} \sum_{k=0}^{\lfloor \frac{t}{2}\rfloor-1} (y_{2k+1}- \widehat{y}_{2k+1})(y_{2k+1} - \widehat{y}_{2k+1})^\top,
\end{align*}
where 
$
\widehat{y}_{2k+1} 
= Ax_{2k} + Bu_{2k} +Av_{2k},
$
and, as a result, 
\begin{align*}
    \widehat{\Sigma}_w
    :=&
    \frac{1}{\lfloor \frac{t}{2}\rfloor}\!\!\! \sum_{k=0}^{\lfloor \frac{t}{2}\rfloor-1}\!\!\! (w_{2k}\!+\!v_{2k+1}\!-\!Av_{2k})(w_{2k} \!+\!v_{2k+1}\!-\!Av_{2k})^{\!\top}\!.
\end{align*}
Note that $\{w_{2k}+v_{2k+1}-Av_{2k}\}_{k=0}^{\lfloor {t}/{2}\rfloor-1}$ is a sequence of independently distributed Gaussian random variables with zero mean and covariance matrix $\Sigma_w+\Sigma_v+A\Sigma_vA^\top$. Hence, $\widehat{\Sigma}_w$ has Wishart distribution $\mathcal{W}_n(\lfloor \frac{t}{2}\rfloor^{-1}(\Sigma_w+\Sigma_v+A\Sigma_vA),\lfloor \frac{t}{2}\rfloor)$.
 This implies that $\widehat{\Sigma}_w$ is a biased estimate of $\Sigma_w$ because
$\mathbb{E}\{\widehat{\Sigma}_w\}=\Sigma_w+\Sigma_v+A\Sigma_vA.$

For the noisy-measurement case, we have the following optimization problem with safety constraints which is to be solved at each  $t > 2n+3 $:
\begin{subequations} \label{optim_prob_zero_mean_noisy_measurement}
\begin{align}
    u_t=&\argmin_{u}  \quad \mathbf{d}(u,\bar{u}_t),\\
    &\quad\, \mathrm{s.t.}  \quad H_{t+1}(Ay_t + Bu +\widecheck{w})\leq h_{t+1}, \\
    & \quad \quad \quad \quad \forall \widecheck{w}: \widecheck{w}^\top \widehat{\Sigma}_w^{-1} \widecheck{w} \leq \frac{\lfloor \frac{t}{2}\rfloor n}{(\lfloor \frac{t}{2}\rfloor-n-1)\delta}.
\end{align}
\end{subequations}

\begin{theorem} \label{thm:safety_prob_noisy_measurement}
Assume that problem~\eqref{optim_prob_zero_mean_noisy_measurement} is feasible and Assumption~\ref{assumption:linear} holds. Then, if we implement the control action $u_t$, $x_{t+1}$ is safe with probability of at least $1-\delta$.
\end{theorem}

\begin{proof} 
Note that $x_{k+1}=Ax_k+Bu_k+w_k=Ay_k+Bu_k+(w_k-Av_k)$. Therefore, we have
\begin{align*}
\mathbb{P}&\{H_{t+1} x_{t+1}\leq h_{t+1}\}  \\&\geq \mathbb{P}\left\{
     (w_t\!-\!Av_t)^\top \widehat{\Sigma}_w^{-1} (w_t\!-\!Av_t) \leq \frac{\lfloor \frac{t}{2}\rfloor n}{(\lfloor \frac{t}{2}\rfloor-n-1)\delta}
     \right\}\notag \\
     & \geq 1-\frac{\mathbb{E}\{ (w_t-Av_t)^\top \widehat{\Sigma}_w^{-1}  (w_t-Av_t)\}}{\lfloor \frac{t}{2}\rfloor n/((\lfloor \frac{t}{2}\rfloor-n-1)\delta)},
\end{align*} 
where, similar to the proof of Theorem~\ref{thm:safety_prob}, the last inequality follows from an application of Markov's inequality. Note that $\widehat{\Sigma}_w[t]^{-1}$ and $w_t-Av_t$ are independent because $\widehat{\Sigma}_w[t]^{-1}$ is only a function of $w_{0},w_2,\dots,w_{2\lfloor {t}/{2}\rfloor-2}$ and  $v_{0},v_1,v_2,\dots,v_{2\lfloor {t}/{2}\rfloor-1}$. Hence, we can compute
\begin{align*}
    \mathbb{E}\{ (w_t-&Av_t)^\top \widehat{\Sigma}_w[t]^{-1}  (w_t-Av_t)\} \\
 =& \trace ( \mathbb{E}\{ (w_t-Av_t)  (w_t-Av_t)^\top\} \mathbb{E}\{\widehat{\Sigma}_w[t]^{-1} \} ) \\
 =&\frac{\lfloor \frac{t}{2}\rfloor \trace((\Sigma_w+A\Sigma_vA)(\Sigma_w+\Sigma_v+A\Sigma_vA)^{-1})}{\lfloor \frac{t}{2}\rfloor-n-1}\\
 \leq &\frac{\lfloor \frac{t}{2}\rfloor n}{\lfloor \frac{t}{2}\rfloor-n-1}.
\end{align*}
This concludes the proof following a similar line of reasoning as in Theorem~\ref{thm:safety_prob}.
\end{proof}

Theorem~\ref{thm:safety_prob_noisy_measurement} shows that access to noiseless state measurements is without loss of generality if the system dynamics are linear, as we can lump the measurement noise and the process noise together. For  nonlinear dynamics, however, passing the effect of the measurement noise through the system dynamics can be challenging.

\subsection{Non-zero mean Gaussian uncertainties}
\label{sec:non_zero_mean}
The assumption of zero mean for the uncertainty $w_t$ may not be appropriate when the uncertainty is not purely random noise but a more systematic uncertainty.
In this section, we assume that $w_t$ is Gaussian with mean $\mu_w$ and covariance $\Sigma_w$, with $\mu_w \neq 0$ in general. We estimate $\mu_w$ and $\Sigma_w$ using 
\begin{align*}
    \widehat{\mu}_w:=&\frac{1}{t} \sum_{k=0}^{t-1} w_k = \frac{1}{t} \sum_{k=0}^{t-1} (x_{k+1} - \widehat{x}_{k+1})\\
    \widehat{\Sigma}_w:=&\frac{1}{t-1} \sum_{k=0}^{t-1} (w_k - \widehat{\mu}_w) (w_k - \widehat{\mu}_w)^\top\\
    =&\frac{1}{t-1} \sum_{k=0}^{t-1} (x_{k+1} - \widehat{x}_{k+1} - \widehat{\mu}_w)  (x_{k+1} - \widehat{x}_{k+1} - \widehat{\mu}_w)^\top,
\end{align*}
where $\widehat{x}_{k+1} = f(x_k) + g(x_k)u_k$.
In contrast to Section \ref{sec:zero_mean}, here $\widehat{\Sigma}_w$ is an unbiased estimate of $\Sigma_w$ when we divide by $t-1$ instead of $t$. This is because we are estimating both the mean and covariance simultaneously. We note that $\widehat{\Sigma}_w$ now has Wishart distribution $ \mathcal{W}_n((t-1)^{-1}\Sigma_w,t-1)$ \cite[Corollary~7.2.3]{anderson2003introduction}.

Letting $\bar{u}_t$ denote the nominal control input, consider now the following optimization problem:
\begin{subequations} \label{optim_prob_nonzero_mean}
\begin{align}
    u_t=&\argmin_{u}  \quad \mathbf{d}(u,\bar{u}_t),\\
    &\quad \,\mathrm{s.t.}   \quad H_{t+1}(f(x_t) + g(x_t) u +w)\leq h_{t+1}, \\
    & \quad \quad \quad \quad \forall w:
     (w-\widehat{\mu}_w)^\top \widehat{\Sigma}_w^{-1} (w-\widehat{\mu}_w) \notag\\
     &\hspace{1.3in}\leq \frac{(t+1)(t-1)n}{t(t-n-2) \delta},
\end{align}
\end{subequations}
which is to be solved at each $t > n+2$.

\begin{theorem}
\label{thm:safety_prob_nonzero_mean}
Assume that problem~\eqref{optim_prob_nonzero_mean} is feasible. Then, if we implement the control action $u_t$, $x_{t+1}$ is safe with probability of at least $1-\delta$.
\end{theorem}

\begin{proof}
We now have 
\begin{align}
     \mathbb{P}\{&H_{t+1} x_{t+1}\leq h_{t+1}\}\notag \\& \geq \mathbb{P}\left\{
      (w-\widehat{\mu}_w)^\top \widehat{\Sigma}_w^{-1}  (w-\widehat{\mu}_w) \leq \frac{(t+1)(t-1)n}{t(t-n-2) \delta}
     \right\}\notag \\
     & \geq 1-\frac{\mathbb{E}\{(w-\widehat{\mu}_w)^\top \widehat{\Sigma}_w^{-1} (w-\widehat{\mu}_w)\}}{(t+1)(t-1)n/(t(t-n-2)\delta)}, \label{eqn:expectation2_0}
\end{align} 
and
\begin{subequations}
\begin{align}
     \mathbb{E}\{(&w-\widehat{\mu}_w)^\top \widehat{\Sigma}_w^{-1} (w-\widehat{\mu}_w)\} \notag\\& = \trace \mathbb{E}\{(w-\widehat{\mu}_w) (w-\widehat{\mu}_w)^\top \widehat{\Sigma}_w^{-1} \} \notag\\
& = \trace \left( \mathbb{E}\{(w-\widehat{\mu}_w) (w-\widehat{\mu}_w)^\top\} \mathbb{E}\{\widehat{\Sigma}_w^{-1} \} \right) \label{eqn:expectation2}\\
& = \trace \left((1+\frac{1}{t}) \Sigma_w \times \frac{t-1}{t-n-2} \Sigma_w^{-1} \right) \label{eqn:expectation2_1} \\
& = \frac{(t+1)(t-1) n}{t(t-n-2)}.  \label{eqn:expectation2_2}
\end{align}
\end{subequations}
The equality in \eqref{eqn:expectation2} follows from independence of $w$ and $ \widehat{\Sigma}_w^{-1}$, together with the property that the sample mean and sample covariance are independent for Gaussian random vectors \cite[Theorem 3.3.2]{anderson2003introduction}. The expression for $\mathbb{E}\{\widehat{\Sigma}_w^{-1} \}$ in the equality in \eqref{eqn:expectation2_1} uses results for the expectation of an inverted Wishart distribution, while the expression for $ \mathbb{E}\{(w-\widehat{\mu}_w) (w-\widehat{\mu}_w)^\top\}$ follows since $w -\widehat{\mu}_w$ has covariance $\Sigma_w + ({1}/{t})\Sigma_w$ when $w$ and $\widehat{\mu}_w$ are independent. 

The result $ \mathbb{P}\{H_{t+1} x_{t+1}\leq h_{t+1}\} \geq 1 - \delta$ follows from substituting~\eqref{eqn:expectation2_2} in~\eqref{eqn:expectation2_0}.
\end{proof}
 
Using Lemma \ref{lemma:robust}, problem (\ref{optim_prob_nonzero_mean}) is reformulated as:
\begin{subequations} \label{optim_prob_reformulation_nonzero_mean}
\begin{align}
    u_t=&\argmin_{u}  \quad \mathbf{d}(u,\bar{u}_t),\\
    &\quad\, \mathrm{s.t.} \quad H_{t+1}(f(x_t) + g(x_t)u + \widehat{\mu}_w )\leq h_{t+1}- \widetilde{e}_{t+1}, 
\end{align}
\end{subequations}
where $\widetilde{e}_{t+1}$ is a vector whose $i$-th entry is equal to
\begin{equation}
\label{eqn:e_vector_nonzero_mean}
    \widetilde{e}_{t+1}^i =\sqrt{\frac{(t+1)(t-1)n}{t(t-n-2)\delta}}\left\| \widehat{\Sigma}_w^{1/2} (H_{t+1}^i)^\top \right\|_2.
\end{equation}

\subsubsection{Safety bounds using Chebyshev's inequality}
We can again derive alternative safety bounds using Chebyshev's inequality. In this case we consider the following problem:
\begin{subequations} \label{optim_prob_chebyshev_nonzero_mean}
\begin{align}
    u_t=&\argmin_{u}  \quad \mathbf{d}(u,\bar{u}_t),\\
    &\quad\, \mathrm{s.t.} \quad H_{t+1}(f(x_t) + g(x_t)u +w)\leq h_{t+1}, \\
    & \quad \quad \forall w:
     \Big|( w-\widehat{\mu}_w)^\top \widehat{\Sigma}_w^{-1}  (w-\widehat{\mu}_w) - \frac{(t+1)(t-1)n}{t(t-n-2)} \Big| \notag \\ & \quad \quad \quad \quad \leq \sqrt{\frac{2  (t+1)^2 (t-1)^2 (t-2) n}{t^2 (t-n-2)^2 (t-n-4)\delta} }, \label{eqn:w_condition_cheybshev_nonzero_mean}
\end{align}
\end{subequations}
which is to be solved at each $t > n+4$.

\begin{theorem} \label{tho:optim_prob_chebyshev_nonzero_mean}
Assume that problem~\eqref{optim_prob_chebyshev_nonzero_mean} is feasible. Then, if we implement the control action $u_t$, $x_{t+1}$ is safe with probability of at least $1-\delta$.
\end{theorem}

\begin{proof}
We have 
\begin{align*}
     &\mathbb{P}\{H_{t+1} x_{t+1}\leq h_{t+1}\} \\ & \geq \mathbb{P}\Bigg\{
      \bigg|( w-\widehat{\mu}_w)^\top \widehat{\Sigma}_w^{-1}  (w-\widehat{\mu}_w) - \frac{(t+1)(t-1)n}{t(t-n-2)} \bigg| \\ & \quad \quad \quad\quad \leq \sqrt{\frac{2  (t+1)^2 (t-1)^2 (t-2) n}{t^2 (t-n-2)^2 (t-n-4)\delta} }
     \Bigg\} \\
     & \geq 1-\frac{t^2 (t-n-2)^2 (t-n-4)\delta}{2  (t+1)^2 (t-1)^2 (t-2) n} \\ & \quad \quad \quad \quad \times \textnormal{Var}\{( w-\widehat{\mu}_w)^\top \widehat{\Sigma}_w^{-1}  (w-\widehat{\mu}_w) \}
\end{align*} 
where the last inequality follows from \eqref{eqn:expectation2_2} and an application of Chebyshev's inequality. As $w-\widehat{\mu}_w \sim \mathcal{N}(0,(1+1/t) \Sigma_w)$ and $\widehat{\Sigma}_w \sim \mathcal{W}_n((t-1)^{-1} \Sigma_w,t-1)$, we can use Lemma \ref{lemma:variance_term} and \eqref{eqn:expectation2_2} to obtain
\begin{align*}
 & \textnormal{Var}\{( w-\widehat{\mu}_w)^\top \widehat{\Sigma}_w^{-1}  (w-\widehat{\mu}_w) \} \\&= \mathbb{E}\left\{\left(( w-\widehat{\mu}_w)^\top \widehat{\Sigma}_w^{-1}  (w-\widehat{\mu}_w)  \right)^2 \right\} \\ & \quad \quad - \left( \mathbb{E} \left\{ ( w-\widehat{\mu}_w)^\top \widehat{\Sigma}_w^{-1}  (w-\widehat{\mu}_w) \right\}\right)^2\\ 
 &= \frac{(t+1)^2(t-1)^2 n(n+2)}{t^2 (t-n-2)(t-n-4)} - \left( \frac{(t+1)(t-1)n}{t(t-n-2)}\right)^2\\
 &=\frac{2  (t+1)^2 (t-1)^2 (t-2) n}{t^2 (t-n-2)^2 (t-n-4)}.    
\end{align*}
Hence $\mathbb{P}\{H_{t+1} x_{t+1}\leq h_{t+1}\}\geq 1-\delta$.
\end{proof}

Let us now denote
\begin{subequations}
\label{eqn:E_V_non_zero_mean}
\begin{align}
\widetilde{E} & := \frac{(t+1)(t-1) n}{t(t-n-2)} \\
\widetilde{V} & :=  \frac{2  (t+1)^2(t-1)^2 (t-2) n}{t^2(t-n-2)^2(t-n-4)}
\end{align}
\end{subequations}
for the mean and variance of $(w - \widehat{\mu})^\top \widehat{\Sigma}_w^{-1} (w - \widehat{\mu})$.
Condition \eqref{eqn:w_condition_cheybshev_nonzero_mean} is equivalent to 
\begin{equation}
\label{eqn:w_condition_equiv_cheybshev_nonzero_mean}
     \forall w: \widetilde{E} - \sqrt{\widetilde{V}/\delta} \leq (w - \widehat{\mu})^\top \widehat{\Sigma}_w^{-1} (w - \widehat{\mu}) \leq \widetilde{E} + \sqrt{\widetilde{V}/\delta}.
\end{equation}
Using arguments similar to that of Theorem~\ref{thm:exact_relax}, instead of  \eqref{optim_prob_chebyshev_nonzero_mean}, we can formulate an equivalent optimization problem which ignores the first inequality in \eqref{eqn:w_condition_equiv_cheybshev_nonzero_mean}:
\begin{subequations} \label{optim_prob2_chebyshev_nonzero_mean}
\begin{align}
    u_t=&\argmin_{u}  \quad \mathbf{d}(u,\bar{u}_t),\\
    &\quad\, \mathrm{s.t.} \quad H_{t+1}(f(x_t) + g(x_t)u +w)\leq h_{t+1}, \\
    & \quad \quad \quad \quad \forall w:
     (w \!- \!\widehat{\mu})^\top \widehat{\Sigma}_w^{-1} (w \!-\! \widehat{\mu})  \leq \widetilde{E} \!+\! \sqrt{\frac{\widetilde{V}}{\delta}}, \label{eqn:w_condition2_cheybshev_nonzero_mean}
\end{align}
\end{subequations}
which is to be solved at each $t > n+4$, where $\widetilde{E}$ and $\widetilde{V}$ are given by \eqref{eqn:E_V_non_zero_mean}.
 Problem \eqref{optim_prob2_chebyshev_nonzero_mean} can be reformulated into the computationally efficient form \eqref{optim_prob_reformulation_nonzero_mean}, with  now
 \begin{equation}
 \label{eqn:e_vector_chebyshev_nonzero_mean}
  \widetilde{e}_{t+1}^i =\sqrt{\widetilde{E} + \sqrt{\widetilde{V}/\delta}}\left\| \widehat{\Sigma}_w^{1/2} (H_{t+1}^i)^\top \right\|_2. \end{equation}




\begin{lemma}
Problem \eqref{optim_prob2_chebyshev_nonzero_mean} is less conservative than problem~\eqref{optim_prob_nonzero_mean} if
${2(t-2)}/({t-n-4}) < {n(1-\delta)^2}/{\delta}.$
In particular,  \eqref{optim_prob2_chebyshev_nonzero_mean} is less conservative than \eqref{optim_prob_nonzero_mean} in the following cases:\\
(i) for all $t > n+4$ if 
   $ \delta < ({2n + 3 - \sqrt{3(n+1)(n+3)}})/{n};
$\\
(ii) for sufficiently large $t$ if 
$
   \delta < ({n+1 - \sqrt{2n+1}})/{n} 
$.
\end{lemma}

\begin{proof}
Similar to the proof of Lemma \ref{lemma:chebyshev_markov_comparison}.
\end{proof}

\subsubsection{Safety bounds valid for all $t \geq 2$}

Under Assumption~\ref{assumption:Sigma_w}, we can provide an alternative robust optimization for maintaining safety while learning,  valid for all $t \geq 2$.

\begin{lemma} \label{lemma:emprical_mean} Under Assumption~\ref{assumption:Sigma_w}, for all $t\geq 1$, 
\begin{align}
    \mathbb{P}\left\{\|\mu_w-\widehat{\mu}_w\|_2^2 \leq \frac{n\sigma}{t\delta}\right\}
    &\geq 1-\delta.  \label{mu_hat_bound}
\end{align}
\end{lemma}

\begin{proof} Note that $\widehat{\mu}_w-\mu_w$ is zero mean Gaussian with covariance $t^{-1}\Sigma_w$. Then 
    $\mathbb{E}\{\|\mu_w-\widehat{\mu}_w\|_2^2\} =\trace(\mathbb{E}\{(\mu_w-\widehat{\mu}_w)(\mu_w-\widehat{\mu}_w)^\top\}) =\trace(t^{-1}\Sigma_w) \leq {n\sigma}/{t}.$ 
From Markov's inequality, we get
\begin{align*}
    \mathbb{P}\left\{\|\mu_w-\widehat{\mu}_w\|_2^2 \leq \varepsilon\right\}
    &\geq 1-\frac{\mathbb{E}\{\|\mu_w-\widehat{\mu}_w\|_2^2\}}{\varepsilon} \geq 1-\frac{n\sigma}{t\varepsilon}.
\end{align*}
Selecting $\delta=n\sigma/(t\varepsilon)$ then gives (\ref{mu_hat_bound}). 
\end{proof}

\begin{lemma} \label{lemma:1_nonzero_mean} Under Assumption~\ref{assumption:Sigma_w}, for all $t\geq 2$, 
\begin{align}
    \mathbb{P}\left\{\Sigma_w\preceq \widehat{\Sigma}_w+ \sigma\sqrt{\frac{ n(n+1)}{\delta (t-1)}} I\right\}\geq 1-\delta. \label{lemma_nonzero_mean_bound}
\end{align}
\end{lemma}

\begin{proof} 
Recall that $\widehat{\Sigma}_w \sim \mathcal{W}_n((t-1)^{-1}\Sigma,t-1)$. The proof then follows by using the same arguments as in the proof of Lemma \ref{lemma:alternative}.
\end{proof}

Now, we can adapt~\eqref{optim_prob_nonzero_mean} to:
\begin{subequations} \label{eqn:control_approximation_nonzero_mean}
\begin{align}
    &u_t=\argmin_{u}  \quad \mathbf{d}(u,\bar{u}_t),\\
    &\mathrm{s.t.} \quad H_{t+1}(f(x_t) + g(x_t)u +w)\leq h_{t+1}, \; \forall w\in\Omega,
\end{align}
\end{subequations}
where 
\begin{align*}
    \Omega:=
    &\Bigg\{w\,\bigg|\,(w-\widehat{\mu}_w-a)^\top \Bigg(\widehat{\Sigma}_w+ \sigma\sqrt{\frac{3 n(n+1)}{\delta (t-1)}} I\Bigg)^{-1} \\& \times (w-\widehat{\mu}_w-a)\leq \frac{3n}{\delta}\mbox{ for some } a^\top a\leq \frac{3n\sigma}{t\delta}\Bigg\}.
\end{align*}

\begin{theorem}
\label{thm:safety_prob_nonzero_mean_alternative}
Assume that problem~\eqref{eqn:control_approximation_nonzero_mean} is feasible and Assumption~\ref{assumption:Sigma_w} holds. Then, if we implement the control action $u_t$, $x_{t+1}$ is safe with probability of at least $1-\delta$.
\end{theorem}

\begin{proof} From Lemma~\ref{lemma:1_nonzero_mean}, we know that
$$
\mathbb{P}\left\{\Sigma_w\preceq \widehat{\Sigma}_w+ \sigma\sqrt{\frac{3 n(n+1)}{\delta (t-1)}} I\right\}\geq 1-\frac{\delta}{3},
$$
and from Lemma~\ref{lemma:emprical_mean}, that
$$
    \mathbb{P}\left\{\|\mu_w-\widehat{\mu}_w\|_2^2 \leq \frac{3n\sigma}{t\delta}\right\}\geq 1-\frac{\delta}{3}. 
$$
Let $A$ denote the event that $\Sigma_w\preceq \widehat{\Sigma}_w+\sigma\sqrt{{3 n(n+1)}/({\delta (t-1)})} I$, and $B$ the event that $\|\mu_w-\widehat{\mu}_w\|_2^2 \leq {3n\sigma}/({t\delta})$. If $A$ and $B$ are satisfied, then we have
\begin{align*}
    &\left\{w\,\bigg|\, (w-\mu_w)^\top \Sigma_w^{-1} (w-\mu_w)\leq \frac{3n}{\delta} \right\}\\
    &\subseteq 
    \Bigg\{w\,\bigg|\, (w-\mu_w)^\top \left(\widehat{\Sigma}_w+ \sigma\sqrt{\frac{3 n(n+1)}{\delta (t-1)}} I\right)^{-1} \\ & \quad \quad \times (w-\mu_w)\leq \frac{3n}{\delta} \Bigg\}\\
    &\subseteq 
    \bigcup_{a:\|a\|_2^2\leq {3n\sigma}/({t\delta})}\Bigg\{w\,\bigg|\, (w-\widehat{\mu}_w-a)^\top \\ & \quad \times \left(\widehat{\Sigma}_w+ \sigma\sqrt{\frac{3 n(n+1)}{\delta (t-1)}} I\right)^{-1} (w-\widehat{\mu}_w-a)\leq \frac{3n}{\delta} \Bigg\}\\
    &=\Bigg\{w\,\bigg|\, \exists a:a^\top a\leq \frac{3n\sigma}{t\delta}\bigwedge (w-\widehat{\mu}_w-a)^\top \\& \quad \times \left(\widehat{\Sigma}_w+ \sigma\sqrt{\frac{3 n(n+1)}{\delta (t-1)}} I\right)^{-1} (w-\widehat{\mu}_w-a)\leq \frac{3n}{\delta}  \Bigg\} \\ & = \Omega.
\end{align*}
Therefore,
\begin{align*}
\mathbb{P}\left\{w \in \Omega| A, B\right\} &\geq \mathbb{P}\left\{ (w-\mu_w)^\top \Sigma_w^{-1} (w-\mu_w)\leq \frac{3n}{\delta} \right\} \\ & \geq 1-\frac{\delta}{3}.
\end{align*}
Combining these probability bounds results in
\begin{align*}
     \mathbb{P}\{H_{t+1} x_{t+1}\leq h_{t+1}\}  &\geq \mathbb{P}\left\{w \in \Omega \right\}
     \\ & \geq \mathbb{P}\left\{w \in \Omega | A, B\right\}   \mathbb{P}\{A, B\}
     \\ & = \mathbb{P}\left\{w \in \Omega | A, B\right\}   \mathbb{P}\{A\} \mathbb{P} \{B\}
     \\ & = \Big(1 - \frac{\delta}{3} \Big)^3
     \geq 1- \delta,
\end{align*} 
where we have used the fact $A$ and $B$ are independent, since the sample mean and sample covariance are independent for Gaussian random vectors.
\end{proof}

\begin{remark} Similar to Remark~\ref{remark:what_is_better}, for large $t$, the optimization problem  \eqref{eqn:control_approximation_nonzero_mean} is more conservative in comparison with~\eqref{optim_prob_nonzero_mean}. However, for $t\leq n+2$, we can use the optimization problem  \eqref{eqn:control_approximation_nonzero_mean} while  problem~\eqref{optim_prob_nonzero_mean} is not applicable. 
\end{remark}
In what follows, we show that $\Omega$ is obtained by the Minkowski sum of an ellipsoid and a spherical set. 
\begin{lemma} \label{lemma:minkowski} A positive semi-definite matrix $\Xi$ and a vector $\xi$ exist such that $\Omega=\{w|(w-\xi)^\top \Xi (w-\xi)\leq 1 \}$. 
\end{lemma}
\begin{proof} Define sets 
    $\mathcal{E} :=\{y+\widehat{\mu}_w\,|\,y^\top (\widehat{\Sigma}_w+ \sigma\sqrt{{3 n(n+1)}/({\delta (t-1)})} I)^{-1} y\leq {3n}/{\delta}\},$ $
    \mathcal{B}:=  \{a\,|\, a^\top a\leq {3n\sigma}/({t\delta})\}.
$ 
We first prove that $\Omega=\mathcal{E} \oplus \mathcal{B}$, where $\oplus$ denotes the Minkowski sum of sets. Note that $w\in \Omega$ if and only if  $\exists a \in\mathcal{B}$ such that $w-a \in \mathcal{E}$. This is equivalent to that there exists $a \in\mathcal{B}$ and $z \in\mathcal{E}$ such that $w-a=z$, or that there exists $a \in\mathcal{B}$ and $z \in\mathcal{E}$ such that $w=a+z$. Finally, this is equivalent to $w\in\mathcal{E} \oplus \mathcal{B}$. The rest follows from~\cite[\S 2]{yan2015closed}. 
\end{proof}

Using Lemma~\ref{lemma:minkowski}, we can simplify~\eqref{eqn:control_approximation_nonzero_mean} and, subsequently, cast this as a robust optimization problem with constraint tightening based on Lemma~\ref{lemma:robust}.

\subsubsection{Adaptive selection of optimization problems}
 A possible procedure which adaptively chooses between problems  \eqref{optim_prob_nonzero_mean}, \eqref{optim_prob2_chebyshev_nonzero_mean}, and \eqref{eqn:control_approximation_nonzero_mean}, depending on which problem is less conservative and their range of validity, is summarized in Algorithm~\ref{alg:safe_learning_nonzero_mean}.
 
\begin{algorithm}[t]
\caption{Safe learning and control for non-zero mean Gaussian uncertainties}
\label{alg:safe_learning_nonzero_mean}
\begin{algorithmic}[1]
\For{$t =2,3,\dots$}
    \If{$t \leq n+2$}
        \State Solve problem \eqref{eqn:control_approximation_nonzero_mean}
    \ElsIf{$t=n+3$ or $t=n+4$}
        \State Solve problem \eqref{optim_prob_nonzero_mean} via \eqref{optim_prob_reformulation_nonzero_mean} and \eqref{eqn:e_vector_nonzero_mean}
    \Else
	    \If{$2(t-2)/(t-n-4) < n(1-\delta)^2/\delta$}
	        \State Solve problem \eqref{optim_prob2_chebyshev_nonzero_mean} via \eqref{optim_prob_reformulation_nonzero_mean} and \eqref{eqn:e_vector_chebyshev_nonzero_mean}
	     \Else
	        \State Solve problem \eqref{optim_prob_nonzero_mean} via \eqref{optim_prob_reformulation_nonzero_mean} and \eqref{eqn:e_vector_nonzero_mean}
	   \EndIf
	\EndIf
\EndFor
\end{algorithmic}
\end{algorithm}

\section{Non-Gaussian Uncertainties}\label{sec:nongaussian}
In this section, the analysis of the previous sections is extended to the case where the uncertainties are non-Gaussian. Initially, the scenario where the uncertainties belong to a zero mean distribution is studied. Later, the case where the distribution has a non-zero mean is investigated. 
\subsection{Zero Mean Non-Gaussian Uncertainties}

The proofs of Theorems~\ref{thm:safety_prob} and~\ref{thm:safety_prob_alternative} are based on  Markov's inequality, which does not assume that the additive process noise is Gaussian. The Gaussianity assumption, in the paper, is to ensure that we need to only learn the mean and variance, rather than all higher-order moments. Furthermore, we make use of the fact that the empirical covariance is Wishart distributed, which is only true for Gaussian uncertainties. The ideas of this paper can be adapted to accommodate non-Gaussian noise processes. This is done in the next lemma and the subsequent theorem. But first we introduce an assumption on the fourth moment of the uncertainty distribution.
\begin{assumption}
\label{assum:fourth_moment}
A positive scalar $\zeta$ exists such that $\mathbb{E}\{\|w_i\|_2^4\}\leq \zeta.$
\end{assumption}
\begin{lemma}[Non-Gaussian Noise] 
\label{lemma:non_Gaussian}
Assume that the process noise satisfies Assumption~\ref{assum:fourth_moment}, but is otherwise arbitrary (potentially non-Gaussian). Then for all $t\geq 1$,
\begin{align*}
    \mathbb{P}\left\{\Sigma_w-\widehat{\Sigma}_w\preceq \sqrt{\frac{\zeta}{t\delta}} I\right\}
    &\geq 1-\delta.
\end{align*}
\end{lemma}

\begin{proof}
Noting that $
\widehat{\Sigma}_w
=({1}/{t})\sum_{i}w_iw_i^\top$, we get
\begin{align*}
    \trace((\widehat{\Sigma}_w-\Sigma)^2)
    =&\trace\Bigg(\Bigg(\frac{1}{t}\sum_{i}w_iw_i^\top -\Sigma\Bigg)^2\Bigg)
\\
=&\frac{1}{t^2}\sum_{i,j} \trace(w_iw_i^\top w_jw_j^\top)\\
&-\frac{2}{t}\sum_{i} \trace(w_iw_i^\top \Sigma)+\trace(\Sigma^2).
\end{align*}
Hence,
\begin{align*}
&\mathbb{E}\{\trace((\widehat{\Sigma}_w-\Sigma)^2)\}
\\&=\frac{1}{t^2}\sum_{i,j} \mathbb{E}\{\trace(w_iw_i^\top w_jw_j^\top)\}
-\trace(\Sigma^2)\\
&=\frac{1}{t^2}\sum_{i} \mathbb{E}\{\trace(w_iw_i^\top w_iw_i^\top)\}
-\frac{1}{t}\trace(\Sigma^2) \\ & \leq  \frac{\mathbb{E}\{\|w_i\|_2^4\}}{t}
\leq  \frac{\zeta}{t}.
\end{align*}
The matrix version of Chebyshev's inequality given in Lemma~\ref{lemma:Chebyshev} of Appendix~\ref{appendix:1} results in
\begin{align*}
    \mathbb{P}\{\Sigma_w-\widehat{\Sigma}_w\preceq \varepsilon I\}
    &\geq 1-\trace(\mathbb{E}\{ (\Sigma_w-\widehat{\Sigma}_w)^{2}\varepsilon^{-2}\})\\
    &\geq  1-\frac{\zeta}{t\varepsilon^2}.
\end{align*}
Setting $\delta=\zeta/(t\varepsilon^2)$ concludes the proof. 
\end{proof}

When the additive process noise is non-Gaussian, we can adapt~\eqref{optim_prob_zero_mean} to:
\begin{subequations} \label{eqn:non_Gaussian}
\begin{align}
    &u_t= \argmin_{u}  \quad \mathbf{d}(u,\bar{u}_t),\\
    & \mathrm{s.t.}   \quad H_{t+1}(f(x_t) + g(x_t)u +w)\leq h_{t+1}, \nonumber\\ 
    &\quad \quad \quad\quad\quad \forall w:
     w^\top \left(\widehat{\Sigma}_w+ \sqrt{\frac{2\zeta}{t\delta}} I\right)^{-1} w \leq \frac{2n}{\delta}.
\end{align}
\end{subequations}

The following result holds for an arbitrary and potentially non-Gaussian process noise.
\begin{theorem}
\label{thm:non_Gaussian}
Assume that problem~\eqref{eqn:non_Gaussian} is feasible and Assumption~\ref{assum:fourth_moment} holds. Then, if we implement the control action $u_t$, $x_{t+1}$ is safe with probability of at least $1-\delta$.
\end{theorem}

\begin{proof} Lemma~\ref{lemma:non_Gaussian} implies that 
\begin{align*}
     \mathbb{P}\left\{\Sigma_w\preceq \widehat{\Sigma}_w+\sqrt{\frac{2\zeta}{t\delta}} I\right\}\geq 1-\frac{\delta}{2}.
\end{align*}
Let $A$ denote the event that $\Sigma_w\preceq \widehat{\Sigma}_w+\sqrt{{2\zeta}/({t\delta})} I$. We have
\begin{align*}
&\mathbb{P}\left\{w^\top \left(\widehat{\Sigma}_w+\sqrt{\frac{2\zeta}{t\delta}} I\right)^{-1} w\leq \frac{2n}{\delta} \Bigg| A \right\}
\\ &\hspace{.7in}\geq 
    \mathbb{P}\left\{w^\top \Sigma_w^{-1} w\leq \frac{2n}{\delta}\right\} \geq 1-\frac{\delta}{2},
\end{align*}
where the last inequality follows from an application of Markov's inequality (for scalar random variables); see the proof of Theorem~\ref{thm:safety_prob_alternative}. Therefore, 
\begin{align*}
    \mathbb{P}\{H_{t+1}& x_{t+1}\leq h_{t+1}\} 
     \\ & \geq \mathbb{P}\left\{w^\top\! \left(\widehat{\Sigma}_w+\sqrt{\frac{2\zeta}{t\delta}} I\right)^{-1}\! w\leq \frac{2n}{\delta} \right\} \\
     & \geq \mathbb{P}\left\{w^\top\! \left(\widehat{\Sigma}_w+\sqrt{\frac{2\zeta}{t\delta}} I\right)^{-1} \hspace{-.1in}w\leq \frac{2n}{\delta} \Bigg| A \right\} \mathbb{P}\{A\} 
     \\ & \geq \left(1-\frac{\delta}{2}\right)^2 \geq  1 - \delta. 
\end{align*} 
This concludes the proof.
\end{proof}

\begin{remark}[Price of Non-Gaussianity]
For Gaussian random variables, since $\mathbb{E}\{\|w\|_2^4\} = \mathbb{E}\{\trace(w w^T w w^T)\}$ and $w w^\top \sim \mathcal{W}_n(\Sigma_w,1)$, we can use  \cite[Theorem 3.1]{Haff79} to show that Assumption~\ref{assum:fourth_moment} holds with $\zeta=\sigma^2 n(n+2)$ if Assumption~\ref{assumption:Sigma_w} holds. This renders the optimization problem~\eqref{eqn:non_Gaussian} more conservative than~\eqref{eqn:control_approximation}, because $\{w|w^\top (\widehat{\Sigma}_w+ \sigma \sqrt{{2 n(n+1)}/({\delta t })} I)^{-1} w \leq {2n}/{\delta}\}\subsetneq \{w|w^\top (\widehat{\Sigma}_w+ \sigma \sqrt{{2 n(n+2)}/({\delta t })} I)^{-1} w \leq {2n}/{\delta}\}$. This conservatism is the price of obtaining a more general result.
\end{remark}
\subsection{Non-zero Mean Non-Gaussian Uncertainties} 
For non-zero mean non-Gaussian process noise, in addition to Assumptions~\ref{assumption:Sigma_w} and~\ref{assum:fourth_moment}, we make the following assumption regarding the mean and the third moment of the uncertainties. 
\begin{assumption} 
\label{assumption:nu_third_moment}
Positive scalars $\nu$ and $\kappa$ exist such that $\|\mu_w\|\preceq \nu$ and $\|\mathbb{E}\{ w_i w_i^\top w_i\}\|_2 \leq \kappa$.
\end{assumption}
\begin{lemma} \label{lemma:non_Gaussian_non_zero_mean} Under Assumptions~\ref{assumption:Sigma_w},~\ref{assum:fourth_moment}, and~\ref{assumption:nu_third_moment},  $\forall t\geq 2$,
\begin{align*}
    \mathbb{P}\left\{\Sigma_w-\widehat{\Sigma}_w\preceq \sqrt{\frac{\gamma}{\delta}} I\right\}
    &\geq  1-\delta,
\end{align*}
where
\begin{align}\label{eq:pi}
\gamma:=\frac{\zeta}{t}\!+\!\frac{(t^2\!-\!t\!+\!1)\nu^4\!+\!(t^2\!+\!2t\!-\!1)n\sigma \nu^2\!+\!4t\kappa \nu}{t(t-1)}.
\end{align}
\end{lemma}
\begin{proof} Note that
\begin{align*}
\widehat{\Sigma}_w
=&\frac{1}{t-1}\sum_{i}\left(w_i-\frac{1}{t}\sum_{j}w_j \right)\left(w_i-\frac{1}{t}\sum_{j}w_j \right)^\top \\
=&\frac{1}{t-1}
\sum_{i}w_iw_i^\top 
-\frac{1}{t(t-1)}\sum_{k,\ell}w_kw_\ell^\top.
\end{align*}
We can therefore compute
\begin{align*}
\widehat{\Sigma}_w^2
=&
\Bigg(\frac{1}{t-1}
\sum_{i}w_iw_i^\top 
-\frac{1}{t(t-1)}\sum_{k,\ell}w_kw_\ell^\top  \Bigg)^2\\
=&\frac{1}{(t-1)^2} \sum_{i,j} w_iw_i^\top w_jw_j^\top \!-\!\frac{1}{t(t-1)^2} \sum_{i,k,\ell} w_iw_i^\top w_kw_\ell^\top\\
&-\frac{1}{t(t-1)^2} \sum_{i,k,\ell} w_kw_\ell^\top w_iw_i^\top\\
&+\frac{1}{t^2(t-1)^2} \sum_{i,j,k,\ell} w_kw_\ell^\top w_iw_j^\top\\
=&\frac{1}{t^2} \sum_{i,j} w_iw_i^\top w_jw_j^\top \!-\!\frac{1}{t(t-1)^2} \sum_{i,k\neq \ell} w_iw_i^\top w_kw_\ell^\top  \\
&-\frac{1}{t(t-1)^2} \sum_{i,k\neq \ell} w_kw_\ell^\top w_iw_i^\top
\\
&+\frac{1}{t^2(t-1)^2} \sum_{i\neq j,k\neq \ell} w_kw_\ell^\top w_iw_j^\top\\
=&\frac{1}{t^2} \sum_{i} (w_iw_i^\top)^2+\frac{1}{t^2} \sum_{i\neq j} w_iw_i^\top w_jw_j^\top \\
  &-\frac{1}{t(t-1)^2} \sum_{i,k\neq \ell} w_iw_i^\top w_kw_\ell^\top  \\
&-\frac{1}{t(t-1)^2} \sum_{i,k\neq \ell} w_kw_\ell^\top w_iw_i^\top\\
&+\frac{1}{t^2(t-1)^2} \sum_{i\neq j,k\neq \ell} w_kw_\ell^\top w_iw_j^\top.
\end{align*}
In what follows, we study each term in the above sum separately. The first term is bounded by
$
    \trace(\mathbb{E}\{(w_iw_i^\top)^2\})
    =\mathbb{E}\{(w_i^\top w_i)^2\}
\leq \zeta.
$
The second term is given by
\begin{align*}
    \mathbb{E}\Bigg\{\!\!\sum_{i\neq j} &\trace(w_iw_i^\top w_jw_j^\top)\!\Bigg\}
    \!=\!t(t\!-\!1) \trace((\mu_w\mu_w^\top \!+\!\Sigma_w)^2).
\end{align*}
The third term can be computed by
\begin{align*}
    \mathbb{E}\Bigg\{\sum_{i,k\neq \ell}& \trace(w_iw_i^\top w_kw_\ell^\top)\Bigg\}
    \\=&\sum_{i\neq \ell}\trace(\mathbb{E}\{w_iw_i^\top w_i\}\mu_w^\top)
    \\&+\sum_{i\neq k}\trace(\mathbb{E}\{w_i^\top w_iw_i^\top \}\mu_w)
    \\&+\sum_{i\neq \ell,i\neq k,k\neq \ell}\trace((\mu_w\mu_w^\top +\Sigma_w)\mu_w^\top \mu_w)\\
    = & 2t(t-1)\varpi^\top \mu_w\\
    &+t(t-1)(t-2)\trace((\mu_w\mu_w^\top +\Sigma_w)\mu_w^\top \mu_w),
\end{align*}
where $\varpi$ is the third-order moment of the noise. We know that $\|\varpi\|_2 \leq \kappa$.
The fourth term is the transpose of the third term. Finally, the fifth term is given by
\begin{align*}
\sum_{i\neq j,k\neq \ell} \mathbb{E}&\{\trace(w_kw_\ell^\top w_iw_j^\top)\}
\\
=&\sum_{k\neq \ell} \mathbb{E}\{\trace(w_kw_\ell^\top w_kw_\ell^\top)\}\\
&+\sum_{k\neq j,k\neq \ell,\ell\neq j} \mathbb{E}\{\trace(w_kw_\ell^\top w_kw_j^\top)\}
\\
&+\sum_{k\neq \ell} \mathbb{E}\{\trace(w_kw_\ell^\top w_\ell w_k^\top)\}\\
&+\sum_{k\neq \ell,i\neq k,i\neq \ell} \mathbb{E}\{\trace(w_kw_\ell^\top w_iw_k^\top)\}\\
&+\sum_{i\neq j,i\neq k,j\neq k} \mathbb{E}\{\trace(w_kw_i^\top w_iw_j^\top)\}\\
&+\sum_{i\neq j,k\neq j,i\neq k} \mathbb{E}\{\trace(w_kw_j^\top w_iw_j^\top)\}\\
&+\sum_{i\neq j,k\neq \ell,i\neq k,j\neq k,\ell\neq i,\ell\neq j} \mathbb{E}\{\trace(w_kw_\ell^\top w_iw_j^\top)\}
\\
=&t(t-1)\trace((\Sigma_w+\mu_w\mu_w^\top)^2)\\
&+2t(t-1)(t-2) \trace((\Sigma_w+\mu_w\mu_w^\top) \mu_w\mu_w^\top)\\
&+t(t-1) \trace((\Sigma_w+\mu_w\mu_w^\top))^2\\
&+2t(t-1)(t-2) \trace(\Sigma_w+\mu_w \mu_w^\top) \mu_w^\top \mu_w\\
&+t(t-1)(t-2)(t-3) (\mu_w^\top\mu_w)^2
\end{align*}
where the last equality follows from
\begin{align*}
    \mathbb{E}\{w_i\mu_w^\top w_i\}&=
    \mathbb{E}\{\mathrm{vec}(w_i\mu_w^\top w_i)\}
    =\mathbb{E}\{w_i^\top \otimes w_i \mathrm{vec}(\mu_w^\top)\}\\
    &=\mathbb{E}\{w_i w_i^\top \} \mu_w 
    =(\Sigma_w+\mu_w\mu_w^\top) \mu_w,
\end{align*}
$
        \mathbb{E}\{w_j^\top w_iw_j^\top\}=
    \mathbb{E}\{(w_j w_i^\top w_j)^\top\}
    = \mu_w^\top (\Sigma_w+\mu_w\mu_w^\top),
$ 
and
\begin{align*}
    \mathbb{E}\{(w_j^\top w_i)^2\}
    =&\mathbb{E}\{w_j^\top w_iw_j^\top w_i\}
    =\mathbb{E}\{w_j^\top \mathrm{vec}(w_iw_j^\top w_i)\}\\
    =&\mathbb{E}\{w_j^\top w_i^\top \otimes w_i \mathrm{vec}(w_j^\top)\}\\
    =&\mathbb{E}\{w_j^\top (\Sigma_w+\mu_w\mu_w^\top) w_j\}\\
    =&\trace(\mathbb{E}\{w_jw_j^\top (\Sigma_w+\mu_w\mu_w^\top) \})\\
    =&\trace((\Sigma_w+\mu_w\mu_w^\top)^2).
\end{align*}
Combining these terms and algebraic manipulations yield
\begin{align*}
    \trace(\mathbb{E}\{\widehat{\Sigma}_w^2\})-\trace(\Sigma_w^2)
    &\leq \frac{1}{t}\zeta+\frac{t^2-t+1}{t(t-1)}\nu^4\\
    &+\frac{(t^2+2t-1)n\sigma \nu^2}{t(t-1)} +\frac{4\kappa \nu}{(t-1)}.
\end{align*}
The matrix version of Chebyshev's inequality given in Lemma~\ref{lemma:Chebyshev} results in 
\begin{align*}
    \mathbb{P}\{\Sigma_w-\widehat{\Sigma}_w\preceq \varepsilon I\}
    \geq &1-\trace(\mathbb{E}\{ (\Sigma_w-\widehat{\Sigma}_w)^{2}\varepsilon^{-2}\})\\
    \geq & 1-\frac{\zeta}{t\varepsilon^{-2}}-\frac{(t^2-t+1)\nu^4}{t(t-1)\varepsilon^{-2}}\\
    &-\frac{(t^2+2t-1)n\sigma \nu^2}{t(t-1)\varepsilon^{-2}} -\frac{4\kappa \nu}{(t-1)\varepsilon^{-2}}.
\end{align*}
This concludes the proof.
\end{proof}

For non-Gaussian additive process noise with non-zero mean, we can adapt~\eqref{eqn:non_Gaussian} to:
\begin{subequations} \label{eqn:non_Gaussian:non-zero-mean}
\begin{align}
    &u_t= \argmin_{u}  \quad \mathbf{d}(u,\bar{u}_t),\\
    & \mathrm{s.t.}   \quad H_{t+1}(f(x_t) + g(x_t)u +w)\leq h_{t+1}, \nonumber\\ 
    &\quad \quad \quad\quad\quad \forall w:
     w^\top \left(\widehat{\Sigma}_w+ \sqrt{\frac{\gamma}{\delta}} I\right)^{-1} w \leq \frac{2n}{\delta}.
\end{align}
\end{subequations}

The following result holds for an arbitrary and potentially non-Gaussian process noise.
\begin{theorem}
\label{thm:non_Gaussian:non-zero-mean}
Assume that problem~\eqref{eqn:non_Gaussian} is feasible and Assumption~\ref{assum:fourth_moment} holds. Then, if we implement the control action $u_t$, $x_{t+1}$ is safe with probability of at least $1-\delta$.
\end{theorem}
\begin{proof}
The proof is similar to that of Theorem~\ref{thm:non_Gaussian}, with the exception of using Lemma~\ref{lemma:non_Gaussian_non_zero_mean} rather than Lemma~\ref{lemma:non_Gaussian}.
\end{proof}

Contrary to the zero-mean case, following the bounds in Lemma~\ref{lemma:non_Gaussian_non_zero_mean} can result in a significantly more conservative robust optimization if the process noise happens to be Gaussian. This is because $\lim_{t\rightarrow \infty} \gamma=\nu^4+n\sigma \nu^2$, where $\gamma$ is given by \eqref{eq:pi}, which implies that $\Sigma_w$ can be significantly different from $\widehat{\Sigma}_w$, while, by using Lemma~\ref{lemma:1_nonzero_mean},  $\widehat{\Sigma}_w$ concentrates around $\Sigma_w$ as $t$ tends to infinity.

\section{Gaussian Uncertainties with Spatially-Varying Mean and Covariance}
\label{sec:piecewise_constant}
In practice, the mean and the covariance are rarely fixed across the environment. To overcome this issue, we approximate location-dependent mean $\mu_w(x)$ and covariance $\Sigma_w(x)$ by piecewise constant functions. That is, we partition the space into sufficiently small non-overlapping connected sets $\mathcal{S}_i$ and, in each set, we assume that the mean and the covariance are constant. This is an effective strategy if the location-dependent mean and covariance are continuous functions. Therefore, if $x_k\in \mathcal{S}_i$, the process noise is Gaussian with mean $\mu_w^{(i)}$ and covariance $\Sigma^{(i)}_w$. Similar to the earlier sections, we do not know the exact values of the mean $\mu_w^{(i)}$ and covariance $\Sigma^{(i)}_w$ when deploying the agent.  
We make the following assumption regarding the quality of the piecewise constant approximation.

\begin{assumption} \label{assum:within_region}
For all regions $\mathcal{S}_i$, $\|\mu_w(x)-\mu_w^{(i)}\|_2\leq \varrho_1$ and $\|\Sigma_w(x)-\Sigma_w^{(i)}\|_F\leq \varrho_2$ for some positive constants $\varrho_1$ and $\varrho_2$.
\end{assumption}

In this section, it is further assumed that $\mu_w$ and $\Sigma_w$ are bounded as per Assumptions~\ref{assumption:Sigma_w} and  \ref{assumption:nu_third_moment}. 

In each $\mathcal{S}_i$, based on all the measurements in all time instances (up to time $t$) that we have been inside this region $\forall k: x_k\in\mathcal{S}_i$, we can learn the mean and the covariance of the process noise  according to
\begin{align*}
    \widehat{\mu}_w^{(i)}:=&\frac{1}{t_i} \sum_{k:x_k\in\mathcal{S}_i} (x_{k+1} - \widehat{x}_{k+1}),\\
    \widehat{\Sigma}_w^{(i)}:=&\frac{1}{t_i-1} \!\sum_{k:x_k\in\mathcal{S}_i}\!\!\!\! (x_{k+1}\!\! -\! \widehat{x}_{k+1}\!\! -\! \widehat{\mu}_w^{(i)})  (x_{k+1} \!\!-\! \widehat{x}_{k+1} \!\!-\! \widehat{\mu}_w^{(i)})^\top,
\end{align*}
where $t_i:=\sum_{k=0}^{t-1} \mathds{1}_{x_k\in\mathcal{S}_i}$, i.e., the number of measurements in region $\mathcal{S}_i$.
Note that we do not need to keep a record of all visited locations to implement these estimators. We can use an online estimator of the form:
\begin{align*}
t_i&\leftarrow t_i+1,\\
\widehat{\mu}_w^{(i)} &\leftarrow  \frac{t_i-1}{t_i}\widehat{\mu}_w^{(i)}+\frac{1}{t_i}(x_{k+1} - \widehat{x}_{k+1}) \\
    \widehat{\Sigma}_w^{(i)}&\leftarrow \frac{t_i-2}{t_i-1}\widehat{\Sigma}_w^{(i)}\\
    &\quad \quad +\frac{1}{t_i-1} (x_{k+1}\! - \widehat{x}_{k+1}\! - \widehat{\mu}_w^{(i)}) (x_{k+1}\! - \widehat{x}_{k+1} \!- \widehat{\mu}_w^{(i)})^\top
\end{align*}
if $x_k\in\mathcal{S}_i$. Therefore, we only need to store $n(n+1)/2$ elements of $\widehat{\Sigma}_w^{(i)}$, $n$ elements of $\widehat{\mu}_w^{(i)}$, and $t_i$, totalling $n^2/2+3n/2+1$ scalars for each region. 

\begin{corollary} \label{cor:emprical} The following hold:
\begin{align}
    \mathbb{P}\bigg\{\| \mu_w^{(i)}-\widehat{\mu}_w^{(i)}\|_2^2\leq\frac{\varrho_1^2}{\delta}+\frac{\sigma n}{t^2\delta}\bigg\}
    &\geq 1-\delta,  \forall t_i\geq 1, \label{generic_bound_mean}\\
    \mathbb{P}\left\{\Sigma_w^{(i)}\preceq \widehat{\Sigma}_w^{(i)}+  \bigg(\frac{n\varrho_2+2\nu\varrho_1}{\delta}\bigg)I\right\}&\geq 1-\delta,  \forall t_i\geq 2. \label{generic_bound_covariance}
\end{align}
\end{corollary}

\begin{proof}
The proof follows from the same line of reasoning as in Lemmas~\ref{lemma:emprical_mean} and~\ref{lemma:1_nonzero_mean}.
\end{proof}


Therefore, if $x_k\in S_i$, we can solve the robust optimization~\eqref{eqn:control_approximation_nonzero_mean} to manoeuvre safely. The bounds in Corollary~\ref{cor:emprical} may sometimes be conservative, since they do not use the fact that we have gathered measurements in a surrounding region $\mathcal{S}_j$, and that the changes in the mean $\mu_w^{(i)}$ and covariance $\Sigma^{(i)}_w$  cannot be arbitrarily abrupt. We say that $\mathcal{S}_j$ is a \emph{neighbour} of $\mathcal{S}_i$ if there exists a control action that can move an agent from some point in $\mathcal{S}_i$ to another point in $\mathcal{S}_{i+1}$ in one time step. 

\begin{assumption} \label{assum:small_changes}
For any two neighbouring regions $\mathcal{S}_i$ and $\mathcal{S}_j$, $\|\mu_w^{(i)}-\mu_w^{(j)}\|_2\leq \rho_1$ and $\|\Sigma_w^{(i)}-\Sigma_w^{(j)}\|_F\leq \rho_2 $ for some positive constants $\rho_1$ and $\rho_2$. 
\end{assumption}

\begin{remark}
Instead of the norm bounds in Assumption~\ref{assum:small_changes}, we can use any metric on probability distributions, such as the Wasserstein distance~\cite{givens1984class}.
\end{remark}

\begin{lemma} \label{lemma:improved_mean}
Let $\mathcal{N}_i$ denote any set consisting of the indices of neighbours of $\mathcal{S}_i$ for which $t_j\geq 1$ for all $j\in\mathcal{N}_i$. Denote $m_i=|\mathcal{N}_i|$. Then,
\begin{align*}
  &  \mathbb{P}\Bigg\{\|\mu_w^{(i)}\hspace{-.03in}-\hspace{-.03in}\widehat{\mu}_w^{(i)}\|_2^2 \leq \frac{(m_i+1)(t_i^2\varrho_1^2+\sigma n)}{t_i^2\delta} \bigwedge \\ &\quad \quad \|\mu_w^{(i)}\hspace{-.03in}-\hspace{-.03in}\widehat{\mu}_w^{(j)}\|_2^2 \leq \left(\rho_1+ \sqrt{\frac{(t_j^2\varrho_1^2+\sigma n) (m_i+1)}{t_j^2\delta}}\right)^2,\\
  &\quad \quad\forall j\in\mathcal{N}_i\Bigg\} \geq 1-\delta.
\end{align*}
\end{lemma}

\begin{proof} Note that, for all $j\in\mathcal{N}_i$, $\mathbb{P}\{\|\mu_w^{(j)}-\widehat{\mu}_w^{(j)}\|_2^2 \leq (m_i+1)(\varrho_1^2t_j^2+\sigma n)/(t_j^2\delta)\geq 1- {\delta}/{(m_i+1)}$. Thus, 
\begin{align*}
 \mathbb{P}\Bigg\{\|&\mu_w^{(i)}-\widehat{\mu}_w^{(j)}\|_2^2 \leq \left(\rho_1+ \sqrt{\frac{(t_j^2\varrho_1^2+\sigma n) (m_i+1)}{t_j^2\delta}}\right)^2 \Bigg\} 
\\ & \geq  \mathbb{P}\Bigg\{\|\mu_w^{(i)}-\mu_w^{(j)}\|_2 + \|\mu_w^{(j)} - \widehat{\mu}_w^{(j)}\|_2 \\
&\hspace{1.1in}\leq \rho_1+ \sqrt{\frac{(t_j^2\varrho_1^2+\sigma n) (m_i+1)}{t_j^2\delta}} \Bigg\} 
\\ & \geq  \mathbb{P}\left\{\|\mu_w^{(j)}-\widehat{\mu}_w^{(j)}\|_2 \leq  \sqrt{\frac{(t_j^2\varrho_1^2+\sigma n) (m_i+1)}{t_j^2\delta}} \right\} \\
&\geq 1-\frac{\delta}{m_i+1}.
\end{align*} 
Combining these probability inequalities  with 
$$\mathbb{P}\left\{\|\mu_w^{(i)}-\widehat{\mu}_w^{(i)}\|_2^2 \leq \frac{(m_i+1)(t_i^2\varrho_1^2+\sigma n)}{t_i^2\delta}\right\}\geq 1- \frac{\delta}{m_i+1}$$ concludes the proof. 
\end{proof}

Depending on the values of $m_i$ and $t_j$ for $j\in\mathcal{N}_i$, the bound in Lemma~\ref{lemma:improved_mean} may be tighter than the generic bound (\ref{generic_bound_mean}) in Corollary~\ref{cor:emprical}. 

\begin{lemma} \label{lemma:improved_variance}
Let $\mathcal{N}_i$ denote any set consisting of the indices of neighbours of $\mathcal{S}_i$ for which $t_j\geq 2$ for all $j\in\mathcal{N}_i$. Denote $m_i=|\mathcal{N}_i|$. Then,
\begin{align*}
    \mathbb{P}\Bigg\{&\Sigma_w^{(i)}\preceq \widehat{\Sigma}_w^{(i)}+
    \frac{(n\tilde{\rho}_2+2\nu\tilde{\rho}_1)(m_i+1)}{\delta} I  \\ & \bigwedge \Bigg[\Sigma_w^{(j)}\preceq  \widehat{\Sigma}_w^{(j)}+ \frac{(n\tilde{\rho}_2+2\nu\tilde{\rho}_1)(m_i+1)}{\delta} I \\ & \quad \quad \bigwedge \|\Sigma_w^{(i)}-\Sigma_w^{(j)}\|_F\leq \rho_2 \Bigg] ,\forall j\in\mathcal{N}_i\Bigg\} \geq 1-\delta.
\end{align*}
\end{lemma}

\begin{proof}
The proof follows from a similar reasoning as in the proof of Lemma~\ref{lemma:improved_mean}. 
\end{proof}

\subsection{Merging Regions}
Most often, to get piecewise constant mean and covariance, we need to segment the space into small regions. This increases the memory required for recalling $t_i$, $\widehat{\mu}_w^{(i)}$, and $\widehat{\Sigma}_w^{(i)}$ for all those regions. In this section, we provide a method for merging these regions.

\begin{lemma}
\label{lemma:merging}
Consider regions $\{\mathcal{S}_i\}_{i\in\mathcal{I}}$. Assume that there exists $i_1,i_2\in\mathcal{I}$ such that  
\begin{align*}
&\|\widehat{\mu}_w^{(i_1)}-\widehat{\mu}_w^{(i_2)}\|_2
    +\|\widehat{\mu}_w^{(i_1)}-\widehat{\mu}_w^{(j)}\|_2
    \\
    &\leq\! \rho_1 \!-\!
    2\!\sqrt{\frac{t_{i_1}^2\!\varrho_1^2\!+\!\sigma n}{t_{i_1}^2\delta}}
    \!-\!
    \sqrt{\frac{t_{i_2}^2\!\varrho_1^2\!+\!\sigma n}{t_{i_2}^2\delta}}
    \!-\!
    \sqrt{\frac{t_j^2\!\varrho_1^2\!+\!\sigma n}{t_j^2\delta}},\!\forall j\!\in\!\mathcal{N}_{i_1},\\
&\|\widehat{\mu}_w^{(i_1)}-\widehat{\mu}_w^{(i_2)}\|
    +\|\widehat{\mu}_w^{(i_2)}-\widehat{\mu}_w^{(j)}\|
    \\
    &\leq\! \rho_1 
    \!-\!
    \sqrt{\frac{t_{i_1}^2\!\varrho_1^2\!+\!\sigma n}{t_{i_1}^2\delta}}
    \!-\!
    2\!\sqrt{\frac{t_{i_2}^2\!\varrho_1^2\!+\!\sigma n}{t_{i_2}^2\delta}}
    \!-\!
    \sqrt{\frac{t_j^2\!\varrho_1^2\!+\!\sigma n}{t_j^2\delta}},\!\forall j\!\in\!\mathcal{N}_{i_2},
\end{align*}
and
\begin{align*}
\|\widehat{\Sigma}_w^{(i_1)}-&\widehat{\Sigma}_w^{(i_2)}\|_F
    +\|\widehat{\Sigma}_w^{(i_1)}-\widehat{\Sigma}_w^{(j)}\|_F
    \\&\leq \rho_2-\frac{8n(n\varrho_2+2\nu\varrho_1)}{\delta},\forall j\in\mathcal{N}_{i_1},\\
\|\widehat{\Sigma}_w^{(i_1)}-&\widehat{\Sigma}_w^{(i_2)}\|_F
    +\|\widehat{\Sigma}_w^{(i_2)}-\widehat{\Sigma}_w^{(j)}\|_F
    \\&\leq \rho_2-\frac{8n(n\varrho_2+2\nu\varrho_1)}{\delta},\forall j\in\mathcal{N}_{i_2}.
\end{align*}
Then, with probability of at least $1-16\delta$, 
\begin{subequations}
\begin{align}
    \|\mu_w^{(i_2)}-\mu_w^{(j)}\|_2 \leq \rho_1,\qquad \forall j\in\mathcal{N}_{i_1},\label{eqn:1:lemma:merging}\\
    \|\mu_w^{(i_1)}-\mu_w^{(j)}\|_2 \leq \rho_1,\qquad \forall j\in\mathcal{N}_{i_2},\label{eqn:2:lemma:merging}\\
    \|{\Sigma}_w^{(i_2)}-{\Sigma}_w^{(j)}\|_F
    \leq \rho_2,\qquad\forall j\in\mathcal{N}_{i_1},\label{eqn:3:lemma:merging}\\
    \|{\Sigma}_w^{(i_1)}-{\Sigma}_w^{(j)}\|_F
    \leq \rho_2,\qquad\forall j\in\mathcal{N}_{i_2}.\label{eqn:4:lemma:merging}
\end{align}
\end{subequations}
\end{lemma}

\begin{proof}
Note that 
\begin{align*}
    \|\mu_w^{(i_1)}&-\mu_w^{(i_2)}\|_2+\|\mu_w^{(i_1)}-\mu_w^{(j)}\|_2\\
    \leq& \|\widehat{\mu}_w^{(i_1)}-\widehat{\mu}_w^{(i_2)}\|_2
    +\|\widehat{\mu}_w^{(i_1)}-\widehat{\mu}_w^{(j)}\|_2+2\|\mu_w^{(i_1)}-\widehat{\mu}_w^{(i_1)}\|_2\\
    & +\|\mu_w^{(i_2)}-\widehat{\mu}_w^{(i_2)}\|_2+\|\mu_w^{(j)}-\widehat{\mu}_w^{(j)}\|_2\\
    \leq & \|\widehat{\mu}_w^{(i_1)}-\widehat{\mu}_w^{(i_2)}\|_2
    +\|\widehat{\mu}_w^{(i_1)}-\widehat{\mu}_w^{(j)}\|_2\\
    &\!+\!2\sqrt{\frac{t_{i_1}^2\varrho_1^2+\sigma n}{t_{i_1}^2\delta}}
    \!+\!
    \sqrt{\frac{t_{i_2}^2\varrho_1^2+\sigma n}{t_{i_2}^2\delta}}
    \!+\!
    \sqrt{\frac{t_j^2\varrho_1^2+\sigma n}{t_j^2\delta}}
    \leq \rho_1
\end{align*}
with probability of at least $1-4\delta$. Similarly, with probability of at least $1-4\delta$, \begin{align*}
    \|\mu_w^{(i_1)}&-\mu_w^{(i_2)}\|_2+\|\mu_w^{(i_2)}-\mu_w^{(j)}\|_2\\
    \leq& \|\widehat{\mu}_w^{(i_1)}-\widehat{\mu}_w^{(i_2)}\|_2
    +\|\widehat{\mu}_w^{(i_1)}-\widehat{\mu}_w^{(j)}\|_2\\
    &
    \!+\!
    \sqrt{\frac{t_{i_1}^2\varrho_1^2+\sigma n}{t_{i_1}^2\delta}}
    \!+\!
    2\sqrt{\frac{t_{i_2}^2\varrho_1^2+\sigma n}{t_{i_2}^2\delta}}
    \!+\!
    \sqrt{\frac{t_j^2\varrho_1^2+\sigma n}{t_j^2\delta}}
    \leq \rho_1.
\end{align*}
This concludes the proof of~\eqref{eqn:1:lemma:merging} and~\eqref{eqn:2:lemma:merging}. 

Because 
$
    \|\Sigma_w^{(i)}- \widehat{\Sigma}_w^{(i)}\|_{F}^2
    =\trace((\Sigma_w^{(i)}- \widehat{\Sigma}_w^{(i)})^2)
    \leq  \trace(\Sigma_w^{(i)}- \widehat{\Sigma}_w^{(i)})^2
$, we have
\begin{align}
    \mathbb{P}
    \bigg\{\|&\Sigma_w^{(i)}\!-\! \widehat{\Sigma}_w^{(i)}\|_{F}^2\leq \bigg( \frac{2n(n\varrho_2+2\nu\varrho_1)}{\delta}\bigg)^2\bigg\}\notag\\
    &\geq 
    \mathbb{P}
    \bigg\{\trace(\Sigma_w^{(i)}\!-\! \widehat{\Sigma}_w^{(i)})^2\leq \bigg(\frac{2n(n\varrho_2+2\nu\varrho_1)}{\delta}\bigg)^2\bigg\}\notag\\
    &=
    \mathbb{P}
    \bigg\{|\trace(\Sigma_w^{(i)}\!-\! \widehat{\Sigma}_w^{(i)})|\leq \frac{2n(n\varrho_2+2\nu\varrho_1)}{\delta}\bigg\}
    \geq 1-\delta,\notag
\end{align}
where the last inequality follows from Lemma~\ref{lemma:varying_distribution}. 
We can use this inequality to establish that
\begin{align*}
  &  \|\Sigma_w^{(i_1)}-\Sigma_w^{(i_2)}\|_F+\|\Sigma_w^{(i_1)}-\Sigma_w^{(j)}\|_F\\
    &\leq \|\widehat{\Sigma}_w^{(i_1)}-\widehat{\Sigma}_w^{(i_2)}\|_F
    +\|\widehat{\Sigma}_w^{(i_1)}-\widehat{\Sigma}_w^{(j)}\|_F   +2\|\Sigma_w^{(i_1)}-\widehat{\Sigma}_w^{(i_1)}\|_F\\
    & \quad +\|\Sigma_w^{(i_2)}-\widehat{\Sigma}_w^{(i_2)}\|_F+\|\Sigma_w^{(j)}-\widehat{\Sigma}_w^{(j)}\|_F\\
    &\leq  \|\widehat{\Sigma}_w^{(i_1)}-\widehat{\Sigma}_w^{(i_2)}\|_F
    +\|\widehat{\Sigma}_w^{(i_1)}-\widehat{\Sigma}_w^{(j)}\|_F
    +\frac{8n(n\varrho_2+2\nu\varrho_1)}{\delta}
    \\ &\leq \rho_2
\end{align*}
with probability of at least $1-4\delta$. Similarly, with probability of at least $1-4\delta$, 
\begin{align*}
  &  \|\Sigma_w^{(i_1)}-\Sigma_w^{(i_2)}\|_F+\|\Sigma_w^{(i_2)}-\Sigma_w^{(j)}\|_F\\
   & \leq  \|\widehat{\Sigma}_w^{(i_1)}-\widehat{\Sigma}_w^{(i_2)}\|_F
    +\|\widehat{\Sigma}_w^{(i_2)}-\widehat{\Sigma}_w^{(j)}\|_F
    +\frac{8n(n\varrho_2+2\nu\varrho_1)}{\delta}
\\ &    \leq \rho_2.
\end{align*}
This concludes the proof.
\end{proof}
Assume that we want to merge regions $\mathcal{S}_{i_1}$ and $\mathcal{S}_{i_1}$. Then, the set of regions becomes $\{\mathcal{S}_i\}_{i\in\mathcal{I}}\cup\{\mathcal{S}_{i_1i_2}\}\setminus\{\mathcal{S}_{i_1},\mathcal{S}_{i_1}\}$ with $\mathcal{S}_{i_1i_2}=\mathcal{S}_{i_1}\cup \mathcal{S}_{i_2}$ denoting the merged region. Note that $\mathcal{N}_{ij}=\mathcal{N}_i\cup\mathcal{N}_j$. For this new region, we have
\begin{align*}
t_{i_1i_2}&\leftarrow t_{i_1}+t_{i_2},\\
\widehat{\mu}_w^{(i_1i_2)} &\leftarrow  \frac{t_{i_1}}{t_{i_1}+t_{i_2}}\widehat{\mu}_w^{(i_1)}+\frac{t_{i_2}}{t_{i_1}+t_{i_2}}\widehat{\mu}_w^{(i_2)} \\
    \widehat{\Sigma}_w^{(i_1i_2)}&\leftarrow \frac{t_{i_1}}{t_{i_1}+t_{i_2}}\widehat{\Sigma}_w^{(i_1)}+\frac{t_{i_2}}{t_{i_1}+t_{i_2}}\widehat{\Sigma}_w^{(i_2)}.
\end{align*}
Evidently, $\widehat{\mu}_w^{(i_1i_2)} $ is a convex combination of $\widehat{\mu}_w^{(i_1)} $ and $\widehat{\mu}_w^{(i_2)} $. Similarly, $\widehat{\Sigma}_w^{(i_1i_2)} $ is a convex combination of $\widehat{\Sigma}_w^{(i_1)} $ and $\widehat{\Sigma}_w^{(i_2)} $. 

\begin{proposition} \label{prop:merge}
Consider regions $\{\mathcal{S}_i\}_{i\in\mathcal{I}}$. 
Assume that 
\begin{enumerate}
    \item $\mu_w^{(i_1i_2)}=\alpha\mu_w^{(i_1)}+(1-\alpha)\mu_w^{(i_2)} $ for some $\alpha\in[0,1]$;
    \item $\Sigma_w^{(i_1i_2)} =\beta\Sigma_w^{(i_1)} +(1-\beta)\Sigma_w^{(i_2)} $ for some  $\beta\in[0,1]$;
    \item Assumptions of Lemma~\ref{lemma:merging} hold.
\end{enumerate}
Then, with probability of at least $1-16\delta$, the following statements hold:
\begin{enumerate}
    \item $\{\mathcal{S}_i\}_{i\in\mathcal{I}}\cup\{\mathcal{S}_{i_1i_2}\}\setminus\{\mathcal{S}_{i_1},\mathcal{S}_{i_1}\}$ satisfies Assumption~\ref{assum:within_region} with  $\varrho_1+\rho_1,\varrho_2+\rho_2$.
    \item $\{\mathcal{S}_i\}_{i\in\mathcal{I}}\cup\{\mathcal{S}_{i_1i_2}\}\setminus\{\mathcal{S}_{i_1},\mathcal{S}_{i_1}\}$ satisfies Assumption~\ref{assum:small_changes} with $\rho_1,\rho_2$.
\end{enumerate}
\end{proposition}

\begin{proof}
Because the 2-norm is convex,
\begin{align*}
   & \|\mu_w(x)-\mu_w^{(i_1i_2)}\|_2
   \\& \leq \alpha\|\mu_w(x)-\mu_w^{(i_1)}\|_2
    +(1-\alpha) \|\mu_w(x)-\mu_w^{(i_2)}\|_2\\
    &\leq  \alpha\|\mu_w(x)-\mu_w^{(i_1)}\|_2
    +(1-\alpha) \|\mu_w(x)-\mu_w^{(i_1)}\|_2
    \\ & \quad +(1-\alpha) \|\mu_w^{(i_1)}-\mu_w^{(i_2)}\|_2\\
    &\leq  \alpha \varrho_1+(1-\alpha)(\varrho_1+\rho_1)
    \leq \varrho_1+\rho_1.
\end{align*}
Therefore, Assumption~\ref{assum:within_region} for the mean holds for $\{\mathcal{S}_i\}_{i\in\mathcal{I}}\cup\{\mathcal{S}_{i_1i_2}\}\setminus\{\mathcal{S}_{i_1},\mathcal{S}_{i_1}\}$ with $\varrho_1+\rho_1$. The proof for the covariance follows the same line of reasoning. 

If $j\in\mathcal{N}_{i_1}$, Assumption~\ref{assum:small_changes} implies that $\|\mu_w^{(i_1)}-\mu_w^{(j)}\|_2 \leq \rho_1$. Furthermore, Lemma~\ref{lemma:merging} gives that  $\|\mu_w^{(i_2)}-\mu_w^{(j)}\|_2 \leq \rho_1$ with probability of at least $1-16\delta$. Similarly, if $j\in\mathcal{N}_{i_2}$, $\|\mu_w^{(i_2)}-\mu_w^{(j)}\|_2 \leq \rho_1$ and $\|\mu_w^{(i_1)}-\mu_w^{(j)}\|_2 \leq \rho_1$ with probability of at least $1-16\delta$. Because the 2-norm is convex, we have
$\|\mu_w^{(i_1i_2)}-\mu_w^{(j)}\|_2 \leq \alpha \|\mu_w^{(i_1)}-\mu_w^{(j)}\|_2+(1-\alpha)\|\mu_w^{(i_1)}-\mu_w^{(j)}\|_2 \leq \rho_1$ for all $j\in\mathcal{N}_{i_1i_2}$. Therefore, Assumption~\ref{assum:small_changes} for the mean holds for $\{\mathcal{S}_i\}_{i\in\mathcal{I}}\cup\{\mathcal{S}_{i_1i_2}\}\setminus\{\mathcal{S}_{i_1},\mathcal{S}_{i_1}\}$ with $\rho_1$. The proof for the covariance follows the same line of reasoning. 
\end{proof}

\section{Numerical Example}
\label{sec:numerical}
In this section we provide a numerical example for some of the results developed in this paper. Consider a robot in 2D space with the state $x = [ x^1\  x^2\  x^3]^{\top}\in\mathbb{R}^3$ composed of its Cartesian coordinates $x^1$ and $x^2$ coordinates and the yaw angle $x^3$. A control input $u = [u^1\ u^2]^{\top}\in\mathbb{R}^2$ composed of the forward speed $u^1\in\mathbb{R}$ and a yaw rate $u^2\in\mathbb{R}$ is applied to the robot with unicycle kinematics:
        $x_{t+1} = x_t + g(x_t)u_t + w_t$ where $g(x_t) = \Delta \begin{bmatrix} \cos(x_3) & 0\\
        \sin(x_3) & 0\\
        0 & 1\end{bmatrix}$, $\Delta$ is the time step of the discretisation and $w_t\in\mathbb{R}^3$ captures the uncertainties of the robot model and its environment. Let $\Delta= \,0.1$ and $w_t\sim\; \mathcal{N}(0,\;0.1)$. The vehicle starts at coordinates $(x^1_0=0, x^2_0=0)$ and yaw angle $x^3_0 = -2.4$ radians. A nominal proportional control action is computed such that it navigates the robot towards its destination at $(\bar{x}^1 = 75m, \bar{x}^2=85m)$: \begin{equation}\label{eq:unicyc_nominal_controller}
    \begin{split}
        &\bar{u}_{t} = \Vert \epsilon_t\Vert_2[0.1  \cos(\theta_{\epsilon} - x^3_t),\, 0.03 \sin(\theta_{\epsilon} - x^3_t)\big]^{\top},
    \end{split}
\end{equation}
where $\epsilon_t \coloneqq [\bar{x}^1 - x^1_t,\, \bar{x}^2- x^2_t]^{\top}$ and $\theta_{\epsilon} \coloneqq \arctan2(\epsilon_t)$. 
This control action is oblivious to the uncertainties $w_t$ and possible obstacles that the robot will encounter. Consider circular obstacles $\mathcal{O}_i:=\{p\in\mathbb{R}^2|\|p-c_i\|_2 \leq r_i\}$, $i\in\{1,2,\cdots 10\}$, where the center $c_i\in\mathbb{R}^2$ and the radius $r_i\in\mathbb{R}$ of each obstacle is known.         
We compute a safe control $u_t$ by modifying the nominal control $\bar{u}_t$ in order to guarantee that the robot is safe (collision free) by imposing the following constraints:
\begin{equation}\label{eq:safety_constraint}
        H_{t+1} (x_t + g(x_t)u_t) \leq h_{t+1} - e_{t+1}
\end{equation}
where $H_{t+1}^i\coloneqq [a_i^{\top},\, 0]$, $ a_i\coloneqq c_i - [x^1,\,x^2]^{\top}$, $h^i_{t+1} \coloneqq a_i^{\top} c_i-({r_i + r_o})/{\Vert a_i \Vert }$, $r_o=0.5$m, and the line $a_i^{\top}[x^1, x^2]^{\top}=h^i_{t+1}$ is perpendicular to the line connecting the robot to obstacle $\mathcal{O}_i$ and is at the offset distance $r_i+r_o$ from the center $c_i$ of obstacle $i$. Condition~\eqref{eq:safety_constraint} constrains the robot not to cross this line in the next time step. Note that $e_t$ adjusts the offset distance based on learning the uncertainty $w_t$ and is calculated according to Algorithm~\ref{alg:safe_learning_zero_mean} with $\sigma =1$ and $\delta = 0.3$. We solve the optimization problem using the open source package CVXPY~\cite{diamond2016cvxpy} in order to calculate the safe control action.\footnote{In some instances this reactive obstacle avoidance strategy can stall behind obstacles. To avoid this we utilize a heuristic modification to the optimization problem. Namely, when the equality in condition~\eqref{eq:safety_constraint} is active, a modified nominal control is assigned towards a point (on the perpendicular safety line) either to the left or to the right of the obstacle (depending on which point is less distant to the goal). 
} 
Figure~\ref{fig:trajectories} shows a run of this simulation which is typical of many similar runs we have observed. As can be seen, learning the disturbance renders the solid trajectory more conservative but safer in the sense of collision with obstacles.   

\begin{figure}
    \centering
    \includegraphics[width=\columnwidth]{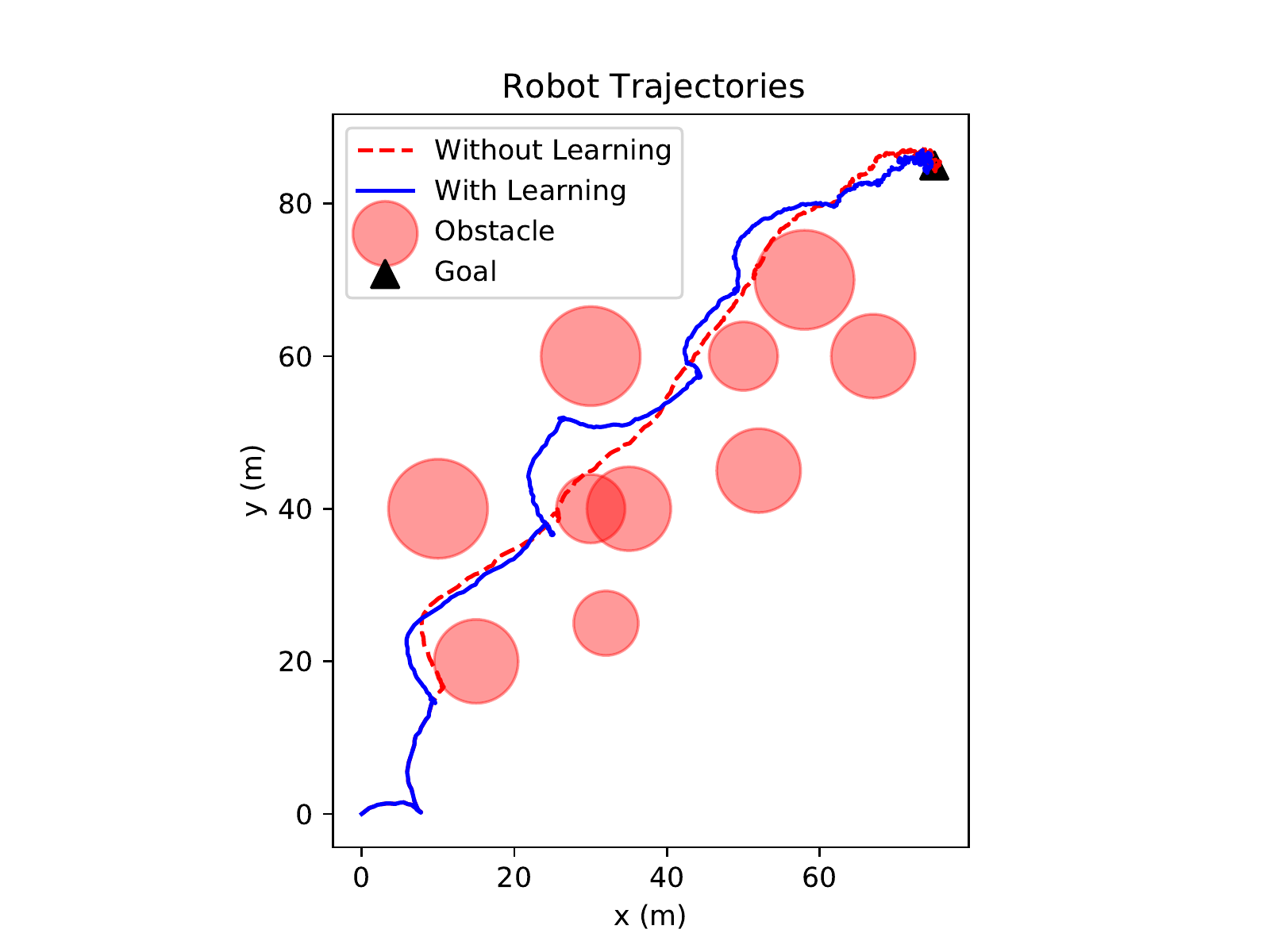}
    \caption{Trajectory of a robot navigating through obstacles with and without considering the disturbance bound.}
    \label{fig:trajectories}
\end{figure}
\section{Conclusions}
We considered the problem of safely navigating a nonlinear control-affine system subject to unknown additive uncertainties. We focused on guaranteeing safety while learning and control proceed simultaneously.  We modelled uncertainty as additive noise with unknown mean and covariance. Subsequently, we used state measurements to learn the mean and covariance of the process noise. We provided rigorous time-varying confidence intervals on the empirical mean and covariance of the uncertainty. This allowed us to employ the learned moments of the noise along with the mentioned confidence bounds to modify the control input via a robust optimization program with safety constraints encoded in it. We proved that we can guarantee that the state will remain in the safe set with an arbitrarily large probability while learning and control are carried out simultaneously. We provided a secondary formulation based on tightening the safety constraints to counter the uncertainty about the learned mean and covariance. The magnitude of the tightening can be decreased as our confidence in the learned mean and covariance increases. Finally, noting that in most realistic cases the environment and its effect on our agent changes within the space, we extended our framework to admit  uncertainties with spatially-varying mean and covariance. 

Future work can focus on a number of important avenues. First, for the non-Gaussian process noise with potentially non-zero mean, the robust formulation can be rather conservative (compared to the bounds in the Gaussian case). An important direction is to improve the confidence bound on learning the parameters of the noise in this case and thus improving the robust formulation.   Second, we consider  continuous mean and covariance noise models that are approximated with piece-wise constant functions. Another future direction is to learn more general nonlinear noise models, for instance a polynomial noise model. Also, appropriate online algorithms for partitioning the space can be developed based on the real-time observations of the realization of the noise. An important direction for future work is to consider output measurements instead of state measurements. Finally, the current results  will be validated in the future on a real robot in an unknown environment. 


\bibliography{IEEEabrv,barrier_function}
\bibliographystyle{IEEEtran} 

\begin{appendix}

\subsection{Matrix concentration inequalities}
\label{appendix:1}

\begin{lemma}[Matrix Markov's inequality]
\label{lemma:matrix_Markov_inequality}
Let $A\succ 0$, and let $X$ be a random matrix such that $X\succeq 0$ almost surely. Then, $\mathbb{P}\{X\preceq A\}\geq1- \trace(\mathbb{E}\{X\}A^{-1}).$
\end{lemma}

\begin{proof}
Note that $\mathbb{P}\{X\succeq A\}=\mathbb{P}\{A^{-1/2}XA^{-1/2}\succeq I\}$. We have $\mathbb{P}\{A^{-1/2}XA^{-1/2}\succeq I\}\leq \lambda_{\min}(A^{-1/2}XA^{-1/2})$~\cite[Corollary~3.3]{jensen1981markov}, where $\lambda_{\min}$ denotes the smallest eigenvalue. As a result, $\mathbb{P}\{A^{-1/2}XA^{-1/2}\succeq I\}\leq \trace(A^{-1/2}XA^{-1/2})=\trace(XA^{-1})$. 
\end{proof}

\begin{lemma}[Matrix Chebyshev's inequality]
\label{lemma:Chebyshev}
Let $A\succ 0$, and let $X$ be a random matrix such that $X\succeq 0$ almost surely. Then, $\mathbb{P}\{X\preceq A\}\geq1- \trace(\mathbb{E}\{X^2\}A^{-2}).$
\end{lemma}

\begin{proof}
First, note that 
$
    \{X\succeq 0| X^2\preceq A^2\}\subseteq \{X\succeq 0| X\preceq A\}
$
because $X^2\preceq A^2$ implies $X\preceq A$~\cite[Proposition 1.2.9]{bhatia2015positive}.
Therefore, 
$\mathbb{P}\{ X\preceq A\}\geq \mathbb{P}\{X^2\preceq A^2\}$. The rest follows from an application of Lemma \ref{lemma:matrix_Markov_inequality} on $\mathbb{P}\{X^2\preceq A^2\}$.
\end{proof}

\subsection{Proof of Lemma \ref{lemma:variance_term}}
\label{appendix:variance_proof}
We will first prove that 
$$\mathbb{E}\{(x^\top A^{-1} x)^2\} = \frac{s_1^2}{s_2^2} \frac{n(n+2)}{(\tau-n-1)(\tau-n-3)}$$
when $x\sim \mathcal{N}(0, s_1 I)$ and $A \sim \mathcal{W}_n(s_2 I, \tau)$, before treating the case for general $\Sigma$. 

Let $B := A^{-1}$, and denote the $ij$-th entry of $B$ by $b_{ij}$. We have 
$
\mathbb{E}\{(x^\top A^{-1} x)^2\} = \mathbb{E}\{(\sum_{i,j} x_i b_{ij} x_j) (\sum_{k,l} x_k b_{kl} x_l) \} 
= \sum_{i,j,k,l} \mathbb{E}\{b_{ij} b_{kl}\} \mathbb{E}\{x_i x_j x_k x_l \}.
$
Note that for all $i$, we have $\mathbb{E}\{x_i\} = 0$, $\mathbb{E}\{x_i^2\} = s_1$,  $\mathbb{E}\{x_i^3\} = 0$,  $\mathbb{E}\{x_i^4\} = 3 s_1^2$. It is then easy to see that $\mathbb{E}\{x_i x_j x_k x_l \} = 0$ unless either 1) $i=j=k=l$, 2) $i=j, k=l, i \neq k$, 3) $i=k, j=l, i \neq j$, or 4) $i=l, j=k, i \neq j$. 
Then 
$
    \sum_{i,j,k,l} \mathbb{E}\{b_{ij} b_{kl}\} \mathbb{E}\{x_i x_j x_k x_l \} 
 = \sum_i 3 s_1^2 \mathbb{E}\{ b_{ii}^2\} + \sum_{i \neq k} s_1^2 \mathbb{E}\{ b_{ii} b_{kk}\} + \sum_{i \neq j}s_1^2 \mathbb{E}\{ b_{ij} b_{ij} \}+ \sum_{i \neq j} s_1^2 \mathbb{E}\{ b_{ij} b_{ji}\}  = \sum_i 3 s_1^2 \mathbb{E}\{ b_{ii}^2\} + \sum_{i \neq k} s_1^2 \mathbb{E}\{ b_{ii} b_{kk}\} + 2 \sum_{i \neq j} s_1^2 \mathbb{E}\{ b_{ij}^2\} 
$
Using the relation $\mathbb{E}\{b_{ij} b_{kl}\} = \textnormal{cov}(b_{ij},b_{kl}) + \mathbb{E}\{b_{ij}\} \mathbb{E}\{b_{kl}\}$  and Theorem 3.2 of \cite{Haff79},  we  derive after some algebra that for $\tau > n+3$:
\begin{align*}
  \mathbb{E}\{ b_{ii}^2\}  & = \frac{1}{s_2^2 (\tau-n-1)(\tau-n-3)} \\
   \mathbb{E}\{ b_{ii} b_{kk}\} & = \frac{\tau-n-2}{s_2^2(\tau-n)(\tau-n-1)(\tau-n-3)} \textnormal{ if } i \neq k \\
   \mathbb{E}\{ b_{ij}^2\}  & = \frac{1}{s_2^2(\tau-n)(\tau-n-1)(\tau-n-3)} \textnormal{ if } i \neq j.
\end{align*}
Hence
\begin{align*}
 \mathbb{E}\{(x^\top A^{-1} x)^2\} 
  &= \frac{3 s_1^2 n}{s_2^2 (\tau-n-1)(\tau-n-3)} \\ & \quad + \frac{s_1^2 n(n-1)(\tau-n-2)}{s_2^2(\tau-n)(\tau-n-1)(\tau-n-3)} \\ & \quad + \frac{2 s_1^2 n(n-1)}{s_2^2(\tau-n)(\tau-n-1)(\tau-n-3)} \\
 &= \frac{s_1^2}{s_2^2} \frac{n(n+2)}{(\tau-n-1)(\tau-n-3)}.
\end{align*}

For the case of general $\Sigma$, note that we can write 
$
 v^\top \Sigma^{-1} v = (\Sigma^{1/2} x)^\top \widehat{\Sigma}^{-1} \Sigma^{1/2} x = x^\top (\Sigma^{-1/2}  \widehat{\Sigma} \Sigma^{-1/2} )^{-1} x. 
$
By~\cite[Theorem 3.2.5]{Muirhead_multivariate}, 
$\Sigma^{-1/2}  \widehat{\Sigma} \Sigma^{-1/2} \sim \mathcal{W}_n (\Sigma^{-1/2} s_2 \Sigma \Sigma^{-1/2}, \tau) = \mathcal{W}_n(s_2 I, \tau).$
Thus 
$\mathbb{E}\{(v^\top \widehat{\Sigma}^{-1} v)^2\} = \mathbb{E}\{(x^\top A^{-1} x)^2\}$, which concludes the proof.

\subsection{Empirical Gaussian Distribution with Samples of Different Mean and Covariance}

In this section, we study the convergence of the empirical Gaussian distribution with samples of slightly different mean and covariance. Evidently, we cannot learn the exact mean and covariance. However, we can converge to their neighborhood.

\begin{lemma} \label{lemma:varying_distribution}
Assume that $w_i$ is a Gaussian random variable with mean $\mu_i$ and covariance $\Sigma_i$ such that $\|\mu_i-\mu\|_2\leq \tilde{\rho}_1$ and $\|\Sigma_i-\Sigma\|_F\leq \tilde{\rho}_2$ for $1\leq i\leq t$. Construct the empirical mean and covariance $
    \widehat{\mu}:=\frac{1}{t}\sum_{j}w_j$, and
    $\widehat{\Sigma}:=\frac{1}{t-1}\sum_{i}\left(w_i-\widehat{\mu} \right)\left(w_i-\widehat{\mu} \right)^\top$. Then, 
\begin{align*}
    \mathbb{P}\bigg\{\| \widehat{\mu}-\mu\|_2\leq\sqrt{\frac{\tilde{\rho}_1^2}{\delta}+\frac{\sigma n}{t^2\delta}}\bigg\}&\geq 1-\delta, \\
    \mathbb{P}\bigg\{\widehat{\Sigma}\preceq \Sigma+\frac{n\tilde{\rho}_2+2\nu\tilde{\rho}_1}{\delta}I\bigg\}&\geq 1-\delta,\\
    \mathbb{P}\bigg\{|\trace(\widehat{\Sigma}-\Sigma)|\leq \frac{2n(n\tilde{\rho}_2+2\nu\tilde{\rho}_1)}{\delta}\bigg\}&\geq 1-\delta.
\end{align*}
\end{lemma}

\begin{proof}
First, note that
\begin{align*}
\mathbb{P}\{\| \widehat{\mu}-\mu\|_2^2\leq \varepsilon^2\}
&\geq 1-\frac{\mathbb{E}\{\| \widehat{\mu}-\mu\|_2^2\}}{\varepsilon^2}\\
&\geq 1-\frac{\tilde{\rho}_1^2}{\varepsilon^2}-\frac{\sigma n}{t^2\varepsilon^2},
\end{align*}
where the second inequality follows from
\begin{align*}
    \mathbb{E}\{\| \widehat{\mu}&-\mu\|_2^2\}\\
    =&
    \mathbb{E}\bigg\{\bigg(\frac{1}{t}\sum_{j}w_j-\mu\bigg)^\top \bigg(\frac{1}{t}\sum_{j}w_j-\mu\bigg)\bigg\}\\
    =&
    \frac{1}{t^2}\sum_{j,i}\trace(\mathbb{E}\{w_jw_i^\top \})
    -\frac{2}{t}\sum_{j} \mu^\top \mathbb{E}\{w_j\}
    +\mu^\top \mu\\
    =&\frac{1}{t^2}\sum_{j,i} \mu_j\mu_i^\top-\frac{2}{t}\sum_{j} \mu^\top \mu_j
    +\mu^\top \mu+\frac{1}{t^2}\sum_j\trace(\Sigma_j)\\
    =&\left\|\frac{1}{t}\sum_{j}\mu_j-\mu \right\|_2^2+\frac{1}{t^2}\sum_j\trace(\Sigma_j)\\
    \leq & \tilde{\rho}_1^2+\frac{\sigma n}{t^2}.
\end{align*}
For the empirical variance, we have
\begin{align*}
\widehat{\Sigma}
=&\frac{1}{t-1}\sum_{i}\left(w_i-\frac{1}{t}\sum_{j}w_j \right)\left(w_i-\frac{1}{t}\sum_{j}w_j \right)^\top \\
=&\frac{1}{t-1}
\sum_{i}\Bigg[w_iw_i^\top -\frac{2}{t}\sum_{j}w_i w_j^\top 
+\frac{1}{t^2}\sum_{j,k}w_k w_j^\top 
\Bigg]\\
=&\frac{1}{t-1}
\sum_{i}w_iw_i^\top -\frac{1}{t(t-1)}
\sum_{i,j}w_i w_j^\top,
\end{align*}
and, as a result,
\begin{align*}
    \mathbb{E}\{\widehat{\Sigma}-&\Sigma\}\\
    =&\mathbb{E}\Bigg\{\frac{1}{t-1}
\sum_{i}w_iw_i^\top 
-\frac{1}{t(t-1)}\sum_{k,\ell}w_kw_\ell^\top \Bigg\}-\Sigma
\\
=&\frac{1}{t}\sum_{i} (\Sigma_i-\Sigma)+\frac{1}{t}\sum_{i}\mu_i\mu_i^\top -\frac{1}{t(t-1)}\sum_{k\neq \ell}\mu_k\mu_\ell^\top\\
=&\frac{1}{t}\sum_{i} (\Sigma_i-\Sigma)
-\frac{1}{t(t-1)}\sum_{i}\sum_{\ell\neq i}\mu_i(\mu_i-\mu_\ell)^\top.
\end{align*}
Using Lemma~\ref{lemma:matrix_Markov_inequality}, we get
\begin{align*}
   & \mathbb{P}\{\widehat{\Sigma}-\Sigma\preceq \varepsilon I\}
    \\&\geq 1-\frac{\trace(\mathbb{E}\{\widehat{\Sigma}-\Sigma\})}{\varepsilon}\\
    &\geq 1-\sum_{i}\frac{\trace(\Sigma_i-\Sigma)}{t\varepsilon}
    -\sum_{\ell\neq k}\frac{\mu_i(\mu_i-\mu_\ell)^\top}{t(t-1)\varepsilon}\\
    &\geq 1-\frac{n\tilde{\rho}_2+2\nu\tilde{\rho}_1}{\varepsilon},
\end{align*}
where the third inequality follows from
$
    \trace(\Sigma_i-\Sigma)
    \leq |\trace(\Sigma_i-\Sigma)|
    \leq \|\Sigma_i-\Sigma\|_F\|I\|_F
    \leq n\tilde{\rho}_2
$
and
$
    \mu_i(\mu_i-\mu_\ell)^\top
    \leq  |\mu_i(\mu_i-\mu_\ell)^\top|
    \leq \|\mu_i\|_2\|\mu_i-\mu_\ell\|_2
    \leq \|\mu_i\|_2(\|\mu_i-\mu\|_2+\|\mu_i-\mu\|_2)
    \leq 2\nu\tilde{\rho}_1.
$
Note that, similarly, we can show that 
\begin{align*}
    \mathbb{P}\{\Sigma-\widehat{\Sigma}\preceq \varepsilon I\}
    \geq 1-&\frac{\trace(\mathbb{E}\{\Sigma-\widehat{\Sigma}\})}{\varepsilon}
    \geq 1-\frac{n\tilde{\rho}_2+2\nu\tilde{\rho}_1}{\varepsilon}.
\end{align*}
Therefore, 
$
    \mathbb{P}\{\Sigma-\varepsilon I\preceq \widehat{\Sigma}\preceq \Sigma+\varepsilon I\}\geq 
    1-{2(n\tilde{\rho}_2+2\nu\tilde{\rho}_1)}/{\varepsilon}.
$
If $\Sigma-\varepsilon I\preceq \widehat{\Sigma}\preceq \Sigma+\varepsilon I$, $|\trace(\widehat{\Sigma}-\Sigma)|\leq \varepsilon n$, which implies that 
$
        \mathbb{P}\{|\trace(\widehat{\Sigma}-\Sigma)|\leq \varepsilon n\}\geq 
        \mathbb{P}\{\Sigma-\varepsilon I\preceq \widehat{\Sigma}\preceq \Sigma+\varepsilon I\}.$ 
Hence, we have
\begin{align*}
    \mathbb{P}\{|\trace(\widehat{\Sigma}-\Sigma)|\leq \varepsilon n\}\geq 1-&\frac{2(n\tilde{\rho}_2+2\nu\tilde{\rho}_1)}{\varepsilon}.
\end{align*}
This concludes the proof.
\end{proof}

\end{appendix}

\end{document}